\documentclass[runningheads]{llncs}
\usepackage[T1]{fontenc}
\usepackage{graphicx}
\usepackage{thmtools}
\usepackage{thm-restate}
\usepackage{amsmath}
\usepackage{mathtools}      
\usepackage{algorithm}      
\usepackage{algpseudocode}  
\usepackage{hyperref}
\usepackage{cleveref}
\usepackage{stmaryrd}
\usepackage{tikz}          
\usetikzlibrary{patterns}  
\usepackage{booktabs}      
\usepackage{subcaption}
\usepackage{multirow}
\usepackage{todonotes}
\usepackage{amsmath,amssymb,mathtools}
\usepackage{microtype}
\usepackage{booktabs}
\usepackage{hyperref}

\hypersetup{
  colorlinks=true,
  linkcolor=blue,
  citecolor=blue,
  urlcolor=blue
}

\usepackage{algorithm}
\usepackage{bm}

\usepackage{algorithm}
\usepackage[
  noEnd=true,
  indLines=true,
  italicComments=false
]{algpseudocodex}

\tikzset{
  algpxIndentLine/.style={draw=black!35, very thin}
}

\newcommand{\ZN}{\mathbb{Z}}
\newcommand{\RN}{\mathbb{R}}
\newcommand{\NN}{\mathbb{N}}
\newcommand{\pNN}{\mathbb{N}_+}

\newcommand{\gG}{\mathfrak{G}}

\newcommand{\dT}{\mathsf{T}}

\newcommand{\aA}{\mathcal{A}}

\newcommand{\aC}{\mathcal{C}}

\newcommand{\aF}{\mathcal{F}}

\newcommand{\aR}{\mathcal{R}}
\newcommand{\aS}{\mathcal{S}}

\newcommand{\aU}{\mathcal{U}}

\newcommand{\aW}{\mathcal{W}}

\newcommand{\sem}[1]{\llbracket   #1 \rrbracket}

 \newcommand{\atoms}{\mathrm{atoms}}
  \newcommand{\ord}{\mathrm{ord}}

  \newcommand{\Predop}{\mathrm{Pred}}

\newcommand{\temp}{\gamma}

\newcommand{\sinit}{s_{\mathrm{init}}}   

\begin{document}
%

\title{Shielding for Higher-Order Safety}
%
%
\author{Filip Cano \and
Thomas A. Henzinger \and
Konstantin Kueffner}
%
\authorrunning{F. Cano, T. A. Henzinger, and K. Kueffner}
%
\institute{Institute of Science and Technology Austria, 3400 Klosterneuburg, Austria \\
\email{\{filip.cano, tah, konstantin.kueffner\}@ist.ac.at}}
%
\makeatletter
\newcommand{\algline}[1]{%
  \begingroup
  \edef\@currentlabel{\theALG@line}%
  \label{#1}%
  \endgroup
}
\makeatother

\maketitle              
\begin{abstract}
Safety shields are runtime enforcement mechanisms that restrict the actions of a controller to guarantee safety. Classical shields are usually synthesized for state predicates: the current physical state is either safe or unsafe, and the shield disables precisely those actions that can force the system into an unsafe state in the future. In many cyber-physical applications this view is too coarse. A vehicle approaching an obstacle should not only avoid collision, but also respect speed regulations, force limits induced by acceleration, and jerk limits to prevent injuries. From a physical perspective, these requirements are predicated over the derivatives of the state. This paper develops a finite-state safety-game construction for such high-order smoothness constraints. We define differential safety properties using finite differences over a discretized state space, characterise their expressiveness, and reduce shield synthesis to an ordinary safety game over a history state space. We give a synthesis algorithm whose shields store exactly $k$ past states for properties of order $k$ and prove that this memory is necessary.
We describe an iterative synthesis procedure for a maximally permissive shield that operates over hierarchies of derivative constraints. The algorithm solves constraints iteratively in increasing order and uses the solution at each iteration to prune the state space for the next constraint.
This makes shield synthesis more efficient in practice, as the algorithm refrains from exploring large regions of the state space that are known to be unsafe. 
\end{abstract}
\section{Introduction}

It is rarely the fall that kills you; it is the strong forces from the deceleration at the end. 
This principle applies, in less dramatic ways, for cyber-physical systems, in which apparently safe states may be doomed by the dynamics of the system. A car one meter from a wall is safe if it is parked, and lethal at high speeds. Safety is often a predicate over states representing positions in space, but also over derivatives of several orders: speed regulations constrain the first derivative of the position, force limits that a material can handle constrain the second, and comfort and injury thresholds constrain the third (jerk) and fourth (snap) derivatives~\cite{Eager_2016,Pendrill_2020}. 
In this paper, we study how to systematically define and enforce safety constraints over properties defined through discrete derivatives of positional states. We call these \emph{differential safety properties}.


A shield is a runtime enforcement mechanism that can restrict or overwrite the actions proposed by an agent, e.g., a learned policy, with the goal of enforcing a specification while interfering as little as possible with the agent.
A popular approach to shield synthesis is to model the interaction between the agent and its environment as a turn-based game and to compute a maximally permissive winning strategy: whenever such a strategy exists, it implements the minimally interfering shield \cite{bloem2015shield,alshiekh2018safe}. 
Classical shields, however, are synthesized for state predicates, and to include properties like speed and acceleration into their specification, these need to be manually modelled. 
This paper extends game-based shield synthesis to \emph{differential safety properties}. 
The price of this expressiveness is memory: an order-$k$ constraint is a predicate over windows of $k+1$ consecutive states, so it can be naturally expressed -- and the corresponding shield synthesized -- in a safety game over $(k+1)$-tuples of states. We study this reduction, and show that shields can be synthesised using only state histories of size $k$.

This direct synthesis method makes no assumption about the structure of the differential safety property. However, in many cases, a differential safety property is naturally expressed as a conjunction of properties of increasing orders. 
For example, a car may not surpass a certain speed limit (order $k=1$), and at the same time may not experience certain accelerations ($k=2$), jerks ($k=3$), and snaps ($k=4$). We observe that when building a shield, we do not need to care about the acceleration induced in sequences of states that contain unsafe speeds, as these states are, by definition, not part of the safety winning region. 
In general, if a $(k+1)$-tuple of states is already unsafe for a property of order $k$, it already signals that any $(k+2)$-tuples containing it will be unsafe for the property of order $k+1$. In Section~\ref{sec:hierarchical-safety}, we exploit this principle to propose a more efficient synthesis algorithm for these types of properties.


\begin{enumerate}
    \item We introduce differential safety properties, a language for constraints
    over discrete derivatives, and show that order-$k$ properties are exactly
    predicates over windows of $k+1$ states.
    \item We give a baseline synthesis procedure by reducing order-$k$
    differential safety games to ordinary safety games over $(k+1)$-state
    histories.
    \item We introduce direct synthesis, which computes the same maximally
    permissive shield using only $k$-state histories, and prove that this
    memory bound is tight.
    \item We introduce iterative synthesis for hierarchical differential safety
    properties, solving derivative levels in increasing order and pruning losing
    windows between levels.
    \item We evaluate the three synthesis procedures experimentally and show
    that direct and iterative synthesis substantially reduce synthesis cost
    compared with the baseline.
\end{enumerate}

\section{Preliminaries}
In this section, we introduce the well-known formal concepts necessary in the formalization of our differential shielding setting.

\paragraph{Turn-based games.}
A turn-based game is played over a finite game graph $ \gG=(\aS,\aA,\dT, \aF, \sinit)$,
where $\aS$ is a finite set of states, $\aA$ is a finite, non-empty set of actions, $\dT\colon \aS\times \aA \to 2^\aS$ is a non-deterministic transition relation, where every state-action pair has at least one successor, $\aF\subseteq \aS^{\omega}$ is the winning condition, and  $\sinit\in \aS$ is an initial state.
The game is played between a \emph{controller}, who in each turn picks an action, and an \emph{environment}, who resolves the nondeterminism of the transition relation by picking a successor state.
Given $s,s'\in \aS$ and $a\in \aA$, we
use $s\xrightarrow{a} s'$ to denote the transition from $s$ to $s'$ with action $a$, i.e., $s'\in \dT(s,a)$, and $s\to s'$ to denote the transition from $s$ to $s'$ with some action, i.e.,  $\exists a\in \aA$ such that $s\xrightarrow{a}s'$.
We write $\aS^*$ for the set of finite state sequences, $\aS^\omega$ for the set of infinite state sequences, and $\aS^\infty \coloneqq \aS^* \cup \aS^\omega$ for the set of all state sequences. For a sequence $w=s_1s_2\cdots\in \aS^\infty$ and $m,n\in \pNN$ with $m\leq n$ we denote $w_{m:n}\coloneqq s_{\min\{m, |w|\}}, \dots, s_{\min\{n, |w|\}}$ as the infix of $w$ that is contained between $m$ and $n$. A path is an infinite sequence of states $w=s_1s_2s_3\cdots \in \aS^\omega$ such that $s_t\to s_{t+1}$ for all $t\in \pNN$.
A path $w$ is \emph{winning} if $w\in \aF$.

\paragraph{Strategy.}
A strategy is a function $\xi\colon \aS^*\to 2^{\aA}$ that maps every history to a set of permitted actions.
A finite or infinite sequence $w=s_1s_2s_3\cdots \in \aS^\infty$ is consistent with $\xi$ if for all $t<|w|$ there exists $a\in\xi(w_{1:t})$ such that $s_t\xrightarrow{a}s_{t+1}$.
A strategy $\xi$ is winning if $\xi(w)\neq\emptyset$ for every finite sequence $w$ starting in $\sinit$ that is consistent with $\xi$, and every path starting from $\sinit$ consistent with $\xi$ is winning. In particular, a winning strategy always permits at least one action along every play it allows.
A strategy $\xi$ subsumes another strategy $\xi'$, denoted by $\xi'\leq \xi$, if every path consistent with $\xi'$ is also consistent with $\xi$.
A winning strategy $\xi$ is maximally permissive, if it subsumes every winning strategy.
A strategy $\xi$ is \emph{$k$-history dependent} if it depends only on the last $k$ states of the history, and thus can be written as  $\xi\colon \aS^{\leq k} \to 2^\aA$. When $k=1$, we say the strategy is memoryless.

\paragraph{Shields.}
A shield is a runtime enforcement mechanism that restricts or corrects the
actions proposed by a controller in order to enforce the winning condition (see Appendix~\ref{sec:app:definitions} for a formal definition). 
In game-based shielding, a sound shield is obtained from a winning strategy:
at every history, the shield permits only actions that keep the play within the
winning region. A maximally permissive winning strategy is minimally
interfering, because it disables an action exactly when allowing it would make
it impossible to guarantee safety~\cite{bloem2015shield}. 
Hence, we focus on computing maximally permissive winning strategies.

\section{Differential Safety Properties}
\label{sec:dsp}
Conditions on derivatives arise naturally in continuous time and continuous state systems, e.g., to control the curvature of a curve, the limit speed of a car, or to enforce a bounded Lipschitz constant of a dynamical system.
In this section, we define a language to specify derivative constraints for discrete time and discrete state systems, e.g., a discrete abstraction of a continuous system.

\subsection{Discrete Derivatives}
In many applications, the state space is a discretization of some $d$-dimensional set of real numbers, i.e., $\aS\subseteq \RN^d$.
In this setting, a path can be viewed as a discretization of a continuous curve in $\RN^d$.
Hence, the smoothness of the path can be quantified by discrete analogues of the derivatives of the underlying curve.
We fix a time step $\Delta\tau>0$.
We always consider $\aS$ to be finite.

\begin{definition}[Discrete derivative]
Let $w=s_1s_2\cdots\in \aS^\infty$ be a state sequence. For $k\in\NN$,
the $k^{\text{th}}$ discrete derivative of $w$ at index $t>k$ is defined
recursively by
\begin{align}
\partial_t^0(w) \coloneqq s_t,
\qquad \text{and}\qquad
\partial_t^k(w)
\coloneqq
\frac{\partial_t^{k-1}(w)-\partial_{t-1}^{k-1}(w)}{\Delta\tau}.
\end{align}
\end{definition}

\begin{example}
\label{ex:car1}
Consider a car moving on a discretized one-dimensional road
$\aS\subseteq [a,b]\cap\ZN$ between two walls $a$ and $b$, as shown in
Fig.~\ref{fig:car}. We set $\Delta\tau=1$. For a trajectory
$w=s_1s_2\cdots$, position, velocity, acceleration, and jerk are given by
\begin{align*}
\partial_t^0(w)&=s_t,
&
\partial_t^1(w)&=\frac{s_t-s_{t-1}}{\Delta \tau},\\
\partial_t^2(w)&=\frac{s_t-2s_{t-1}+s_{t-2}}{(\Delta \tau)^2},
&
\partial_t^3(w)&=\frac{s_t-3s_{t-1}+3s_{t-2}-s_{t-3}}{(\Delta \tau)^3}.
\end{align*}
\end{example}


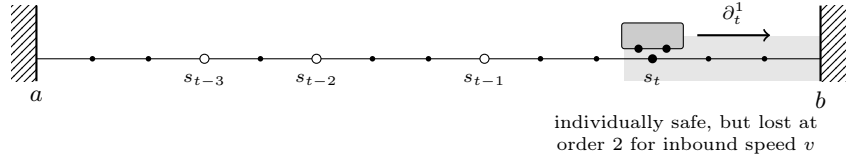
\begin{figure}[t]
\centering
\begin{tikzpicture}[xscale=0.74,yscale=0.74]
  \fill[black!10] (10.5,-0.4) rectangle (14,0.4);
  \node[font=\scriptsize, align=center] at (11.6,-1.35) {individually safe, but lost at\\order $2$ for inbound speed $v$};
  \draw (0,0) -- (14,0);
  \foreach \x in {1,2,...,13} \fill (\x,0) circle (1.4pt);
  \fill[pattern=north east lines] (-0.45,-0.4) rectangle (0,0.95);
  \fill[pattern=north east lines] (14,-0.4) rectangle (14.45,0.95);
  \draw[thick] (0,-0.4) -- (0,0.95);
  \draw[thick] (14,-0.4) -- (14,0.95);
  \node[below,font=\small] at (0,-0.42) {$a$};
  \node[below,font=\small] at (14,-0.42) {$b$};
  \draw[fill=white] (3,0) circle (2.3pt);
  \node[below,font=\scriptsize] at (3,-0.12) {$s_{t-3}$};
  \draw[fill=white] (5,0) circle (2.3pt);
  \node[below,font=\scriptsize] at (5,-0.12) {$s_{t-2}$};
  \draw[fill=white] (8,0) circle (2.3pt);
  \node[below,font=\scriptsize] at (8,-0.12) {$s_{t-1}$};
  \draw[fill=black!20,rounded corners=1pt] (10.45,0.18) rectangle (11.55,0.64);
  \fill (10.75,0.18) circle (2.1pt);
  \fill (11.25,0.18) circle (2.1pt);
  \fill (11,0) circle (2.3pt);
  \node[below,font=\scriptsize] at (11,-0.12) {$s_{t}$};
  \draw[->,thick] (11.8,0.42) -- (13.1,0.42);
  \node[above,font=\scriptsize] at (12.45,0.46) {$\partial^1_t$};
\end{tikzpicture}
\caption{The car of Example~\ref{ex:car1} on a discretized road between two walls $a$ and $b$. The last $k+1$ positions $s_{t-k},\dots,s_t$ form the window on which position ($\partial^0$), velocity ($\partial^1$), acceleration ($\partial^2$), and jerk ($\partial^3$) are evaluated. The shaded positions are individually safe, but every window entering them at high inbound speed is losing once acceleration is bounded: the car cannot brake in time (cf. Ex.~\ref{ex:hierarchical-car-iterative}).}
\label{fig:car}
\end{figure}

Note that discrete derivatives are backward differences, using only past and current states. In general, discrete derivatives permit a closed-form expression. 

\begin{restatable}{lemma}{derivativeClosedForm}
\label{lem:closed-form}
The $k^{\text{th}}$ derivative satisfies the general formula
\begin{align*}
\partial_t^k(w)
=\frac{1}{(\Delta \tau)^k}\sum_{i=0}^{k}(-1)^i\binom{k}{i}s_{t-i} \quad \text{ if $k<t\leq |w|$, and undefined otherwise.}
\end{align*}
\end{restatable}

\subsection{Differential Safety Properties}
Classical safety properties, e.g., defined using LTL, partition the states into a set of safe (or good) states.
We extend this to discrete derivatives
and introduce a specification language for differential safety properties.

\paragraph{Syntax.}
We define differential safety properties syntactically as Boolean expressions over atoms, i.e., expressions $\partial^k \in \aR$ that restrict the derivative of some order $k$ to some safe set $\aR$. Formally, the language is given by the grammar:
\begin{align*}
\varphi \Coloneqq  (\partial^k \in \aR)  \; \text{where} \;  k \in \NN  , \,\aR \subseteq \RN^d \; \mid \; \varphi \land \varphi \; \mid \; \lnot \varphi.
\end{align*}
The order of a differential safety property $\varphi$ is given by the degree of its largest discrete derivative, i.e., $\ord(\varphi)\coloneqq\max\{ k \mid (\partial^k \in \aR)  \in \atoms(\varphi)\}$ where $\atoms(\varphi)$ denotes the set of all atoms in $\varphi$.
Since differential safety properties are closed under Boolean combinations, they can express conditional requirements, e.g., that a car which is close to a wall and still moving towards it must brake sufficiently.

\paragraph{Semantics.}
We consider a local and a global interpretation.
Locally, an order $k$ atom $\partial^k \in \aR$ is interpreted over a sequence $w\in \aS^\infty$ at index $k+1\leq t \leq |w|$ as
\begin{align*}
\sem{\partial^k \in \aR}_t(w) \coloneqq \partial_t^k(w)\in \aR,
\end{align*}
and is extended to Boolean expressions as usual. Moreover, for the complexity statements, we assume that the safe sets $\aR$ are represented such that membership is decided in constant time, e.g., as intervals or finite unions of boxes.
Globally, an expression $\varphi$ of order $k$ is interpreted over a sequence $w\in \aS^\infty$ as
\begin{align*}
\sem{\varphi}(w)\coloneqq  \forall t\in \{k+1,\dots, |w|\} \colon \sem{\varphi}_t(w).
\end{align*}
In particular, sequences of length at most $k$ satisfy $\varphi$ vacuously, as no derivative of order $k$ can yet be evaluated.
We use $w\models\varphi$ to denote $\sem{\varphi}(w)=\top$, and
$w\not\models\varphi$ to denote $\sem{\varphi}(w)=\bot$.


\begin{example}[Ex.~\ref{ex:car1} cont.]
\label{ex:car2}
For the car of Example~\ref{ex:car1}, a joint position, velocity, and acceleration constraint
can be written as
\begin{align*}
    \varphi
    \coloneqq
    \partial^0 \in \aS\setminus\{a,b\}
    \land
    \partial^1 \in [v_{\min},v_{\max}]
    \land
    \partial^2 \in [a_{\min},a_{\max}].
\end{align*}
where the boundary points
$a,b$ are unsafe walls.
At time $3$, this requires
$s_3\in\aS\setminus\{a,b\}$,
$(s_3-s_2)/\Delta\tau\in[v_{\min},v_{\max}]$, and
$(s_3-2s_2+s_1)/(\Delta\tau)^2\in[a_{\min},a_{\max}]$.
\end{example}


\paragraph{Expressiveness.}
By Lemma~\ref{lem:closed-form}, the value of every atom of order at most $k$ is a function of the window $(s_{t-k},\dots,s_t)$ of the last $k+1$ states. The converse also holds: over a finite state space, differential safety properties of order $k$ can express \emph{every} predicate over such windows.

\begin{restatable}[Window predicates]{lemma}{windowPredicates}
\label{lem:windows}
Let $\aS\subseteq\mathbb R^d$ be a state space and $k\in\NN$.
\begin{enumerate}
\item For every $P\subseteq\aS^{k+1}$ there exists a differential safety property $\varphi_P$ of order $k$ such that, for every $w\in\aS^\infty$ and $k+1\leq t\leq|w|$, $\sem{\varphi_P}_t(w)=\top$ iff $(s_{t-k},\dots,s_t)\in P$.
\item Conversely, for every differential safety property $\varphi$ of order $k$, the value $\sem{\varphi}_t(w)$ depends only on the window $(s_{t-k},\dots,s_t)$.
\end{enumerate}
\end{restatable}
Hence differential safety properties of order $k$ specify exactly the predicates over windows of $k+1$ consecutive states. In particular, arbitrary decidable trade-offs between a state and its derivatives are expressible, such as requiring stronger braking the closer the car is to an obstacle.

\section{Shield Synthesis for Differential Safety Properties}
\label{sec:synthesis}
Our objective is to enforce derivative constraints of discrete systems. Hence, we require a shield synthesis algorithm for differential safety properties.
In this section, we define differential safety games, which are turn-based games with a winning condition characterised by a differential safety property,
and provide two algorithms for finding the corresponding  maximally permissive winning strategy.

\begin{definition}[Differential Safety Game]
A $k$-th order \emph{differential safety game} is a turn-based game $\gG_\varphi=(\aS,\aA,\dT, \aF_\varphi, \sinit)$ where $ \aF_\varphi=\{w\in \aS^\omega \colon w\models \varphi\}$ is defined w.r.t.\ a $k$-th order differential safety property $\varphi$.
\end{definition}
Throughout this section, we assume $k=\ord(\varphi)\geq 1$; for $k=0$, the winning condition is a state invariant and the game is a classical safety game.

\begin{example}[Ex.~\ref{ex:car2} cont.]
\label{ex:car-game}
We turn Example~\ref{ex:car2} into a differential safety game. The controller
chooses a displacement $u\in\{-u_{\max},\dots,u_{\max}\}$, the environment
chooses a disturbance $d$ with $|d|\leq d_{\max}<u_{\max}$, and the successor
state is
\begin{align*}
\dT(s,u)
\coloneqq
\{\mathrm{clip}(s+u+d) \mid d\in\ZN,\ |d|\leq d_{\max}\},
\end{align*}
where $\mathrm{clip}(x)=\min\{\max\{x,a\},b\}$. Position safety alone can be
enforced from every interior point of a sufficiently wide road, but adding
velocity and acceleration constraints may make states near a wall losing (see
Example~\ref{ex:hierarchical-car-iterative}).
\end{example}


\subsection{Baseline Synthesis using History Games}
The maximally permissive winning strategy of a safety game is commonly obtained using a fixpoint algorithm \cite{thomas1995synthesis,GradelThomasWilke2002Games}.
However, this requires the winning condition to be a state invariant, which is not the case for $k$-th order differential safety games.
Fortunately, we can recover state invariance by expanding the state space using a standard product game construction.

\paragraph{History games.}
A differential safety property of order $k$ depends only on the current state and
the previous $k$ states, since the value of $\partial_t^k(w)$ is determined by
$s_{t-k},\dots,s_t$. Hence, although the property is not a state invariant on
the original state space $\aS$, it becomes a state invariant once the last
$k+1$ states are stored explicitly. Intuitively, the product construction turns
the finite memory needed to evaluate the derivative constraint into part of the
state.
For a differential safety game
$\gG_\varphi=(\aS,\aA,\dT,\aF_\varphi,\sinit)$ with $\ord(\varphi)=k$, we define
its corresponding history game $ \gG^k_\varphi \coloneqq (\aS^{\leq k+1},\aA,\dT^k,\aF^k_\varphi,(\sinit)).$
The transition relation $\dT^k$ updates the stored history by appending the
new successor state. If the history has not yet reached length $k+1$ the new
state is simply appended, and once length $k+1$ has been reached the oldest state
is discarded. Formally, for $(s_1,\dots,s_j)\in \aS^{\leq k+1}$,
\begin{align*}
\dT^k((s_1,\dots,s_j),a)
=
\begin{cases}
\{ (s_2,\dots,s_j,s_{j+1}) \mid s_{j+1}\in \dT(s_j,a)\}
& \text{if } j=k+1, \\
\{ (s_1,\dots,s_j,s_{j+1}) \mid s_{j+1}\in \dT(s_j,a)\}
& \text{if } j<k+1.
\end{cases}
\end{align*}
The winning region of the history game consists of those histories from which the differential safety property can still be enforced forever. Equivalently,
a path in the history game is winning iff the corresponding projected path in the
original game satisfies $\varphi$. Hence, for every original path
$w=s_1s_2s_3\cdots \in \aS^\omega$, its lifted history path is winning precisely when
\begin{align*}
(w_{\max\{1,t-k\}:t})_{t\in\pNN} \in \aF^k_\varphi
\quad\Longleftrightarrow\quad
w\in \aF_\varphi .
\end{align*}
A memoryless strategy $\xi^k\colon \aS^{\leq k+1}\to 2^\aA$ for
$\gG^k_\varphi$ therefore induces a $(k+1)$-history dependent strategy for the
original game.
The construction is a de Bruijn-graph construction: the saturated history states are the $(k+1)$-grams of states, and transitions shift the gram by one position.
We show that $\aF_\varphi^k$ is a state invariance condition on $\gG_\varphi^k$, which implies that $\gG_\varphi^k$ permits a maximally permissive memoryless strategy.

\begin{restatable}{lemma}{naive}
\label{lem:naive}
Let $\gG_\varphi = (\aS, \aA, \dT, \aF_\varphi, \sinit)$ be a $k$-th order differential safety game.
\begin{enumerate}
\item $\gG_\varphi^k$ is a safety game, i.e., the winning condition $\aF_\varphi^k$ is a state invariance condition on $\aS^{\leq k+1}$.
\item A memoryless strategy is winning for $\gG_\varphi^k$ if and only if the induced $(k+1)$-history dependent strategy is winning for $\gG_\varphi$.
\end{enumerate}
\end{restatable}
We call this construction \emph{baseline synthesis}: it reduces the
differential safety game to an ordinary safety game over histories of length
$k+1$ and then applies the standard safety-game fixpoint algorithm.


\subsection{Optimised Direct Synthesis}
The baseline synthesis stores histories of length $k+1$. We now present an optimised \emph{direct synthesis} algorithm, which computes the same maximally permissive shield
while storing only histories of length $k$.

\paragraph{Direct synthesis.}
Intuitively, Algorithm~\ref{algo:1} operates as follows.
The shield stores histories of length at most $k$. This is sufficient because a
$k$-th order derivative is evaluated on $k+1$ consecutive states, and the
candidate successor state supplies the last state of this window.
The algorithm first permits all actions on histories of length at most $k-1$,
because no $k$-th order derivative can yet be evaluated on such histories
(Lines~\ref{alg:init-short}--\ref{alg:set-short}); recall that, by convention, sequences of length at most $k$ satisfy $\varphi$ vacuously.
It then considers every length-$k$ history. An action is initially considered
safe if every successor state produces a length-$(k+1)$ history satisfying
$\varphi$ at its last index (Lines~\ref{alg:full-hist}--\ref{alg:add-safe}).
Whenever a length-$k$ history has no safe action, it is inserted into the
conflict set $\aU$ (Lines~\ref{alg:empty-check}--\ref{alg:add-conflict}).
After initialisation, unsafe histories are propagated backwards through
\textsc{PropagateConflicts}. If a history is unsafe, then every action leading
to it is removed from each of its predecessor histories
(Lines~\ref{alg:pred-actions}--\ref{alg:remove-action}): a history of length
$1<j<k$ has its prefix $(s_1,\dots,s_{j-1})$ as unique predecessor, while a
length-$k$ history has both its prefix $(s_1,\dots,s_{k-1})$, which occurs at
the start of a run, and the length-$k$ histories $(s',s_1,\dots,s_{k-1})$
preceding it in the sliding window; for $k=1$, the construction specialises to
the classical safety-game algorithm. If a removal makes
the predecessor history unsafe, the predecessor is inserted into $\aU$
(Line~\ref{alg:pred-empty-check}).
A history of length $j=1<k$ marks the start of a run and has no predecessor
history, so nothing is propagated; if $\xi((\sinit))=\emptyset$ upon
termination, no winning strategy exists.
The algorithm terminates because each action can be removed from each history at
most once, and a history is inserted into $\aU$ only when its set of safe actions
becomes empty.

\begin{algorithm}[t]
\begin{algorithmic}[1]
\caption{Direct synthesis algorithm}
\label{algo:1}
\Procedure{MaxPermStrategy}{$\gG,\varphi$}
    \State \textbf{Initialize} $k \gets \ord(\varphi), \xi(\bar s) \gets \emptyset,\: \forall \bar s\in\aS^{\leq k}, \aU \gets \emptyset$ \Comment{safe actions and conflicts}
    \For{$s_1,\dots, s_j\in\aS^{\leq k-1}$} \algline{alg:init-short}
        \State $\xi(s_1,\dots, s_j)\gets \aA$ \Comment{too short to evaluate $\partial^k$} \algline{alg:set-short}
    \EndFor
    \For{$(s_1,\dots, s_k)\in\aS^k$} \algline{alg:full-hist}
        \For{$a\in\{a\in\aA \mid \forall s'\in\dT(s_k,a)\colon \sem{\varphi}_{k+1}(s_1,\dots,s_k,s')=\top\}$}
            \State $\xi(s_1,\dots, s_k)\gets \xi(s_1,\dots, s_k)\cup \{a\}$ \Comment{all successors satisfy $\varphi$} \algline{alg:add-safe}
        \EndFor
        \If{$\xi(s_1,\dots, s_k)=\emptyset$} \algline{alg:empty-check}
            \State $\aU\gets \aU\cup \{(s_1,\dots, s_k)\}$ \Comment{no safe action remains} \algline{alg:add-conflict}
        \EndIf
    \EndFor
    \State $\xi \gets \textnormal{\scshape PropagateConflicts}(\gG,k,\xi,\aU)$
    \State \Return $\xi$
\EndProcedure

\Procedure{PropagateConflicts}{$\gG,k,\xi,\aU$}
    \While{$\aU\neq \emptyset$}
        \State $(s_1,\dots, s_j)\gets \mathrm{Pop}(\aU)$ \Comment{unsafe history}
        \For{$a\in\aA$} \algline{alg:pred-actions}
            \If{$1<j<k$ \textbf{and} $s_{j-1}\xrightarrow{a}s_j$}
                \State $\bar S \gets \{(s_1,\dots, s_{j-1})\}$ \Comment{unique shorter predecessor}
            \ElsIf{$j=k\geq 2$ \textbf{and} $s_{k-1}\xrightarrow{a}s_k$}
                \State $\bar S \gets \{(s_1,\dots, s_{k-1})\}\cup\{(s',s_1,\dots, s_{k-1}) \mid s'\rightarrow s_1\}$ \Comment{short prefix and sliding predecessors}
            \Else $\quad \bar S\gets \emptyset$ \Comment{start of a run, no predecessor}
            \EndIf
            \For{$\bar s\in \bar S$}
                \If{$\xi(\bar s)=\{a\}$} \algline{alg:pred-empty-check}
                     $\aU\gets \aU\cup \{\bar s\}$ \Comment{removing $a$ makes $\bar s$ unsafe} \algline{alg:pred-add-conflict}
                \EndIf
                \State $\xi(\bar s)\gets \xi(\bar s)\setminus \{a\}$ \Comment{avoid forcing an unsafe history} \algline{alg:remove-action}
            \EndFor
        \EndFor
    \EndWhile
    \State \Return $\xi$
\EndProcedure
\end{algorithmic}
\end{algorithm}

\begin{restatable}{theorem}{AlgorithmOne}
\label{thrm:algo1}
Algorithm~\ref{algo:1} produces a $k$-history dependent maximally permissive winning strategy for $\gG_\varphi$, if it exists, and has a time complexity of $\mathcal O(|\aS|^{k+1}\cdot |\aA|)$.
\end{restatable}

\paragraph{Reachable histories.}
Algorithm~\ref{algo:1} enumerates the full history space $\aS^{\leq k}$, although only histories that can actually occur in a play matter for the shield. We call a history $(s_1,\dots,s_j)$ \emph{path-connected} if $s_l\to s_{l+1}$ for all $l<j$; every history arising in a play is path-connected. Restricting the algorithm to path-connected histories preserves its result on all of them.

\begin{restatable}[Reachability pruning]{proposition}{reachabilityPruning}
\label{prop:reachability}
Let $B$ bound the in- and out-degree of the transition relation $\to$, i.e., $|\{s' \mid s\to s'\}|\leq B$ and $|\{s' \mid s'\to s\}|\leq B$ for every $s\in\aS$. Restricted to path-connected histories, Algorithm~\ref{algo:1} computes the same safe action sets on every path-connected history, in time $\mathcal O(|\aS|\cdot B^{k}\cdot|\aA|)$.
\end{restatable}
In discretizations of physical systems, $B$ is governed by the maximal displacement per time step and is typically much smaller than $|\aS|$. Moreover, the differential constraints themselves reduce the effective degree: a speed bound $\partial^1\in[v_{\min},v_{\max}]$ leaves at most $(v_{\max}-v_{\min})\Delta\tau+1$ grid successors per state, so reachability pruning composes particularly well with the iterative synthesis of Section~\ref{sec:hierarchical-safety}.

\subsection{Memory Requirements (Hardness)}
Next we show that shields for differential safety properties of order $k$ do require at least $k$-history dependent strategies.

\paragraph{Hardness.}
We can demonstrate hardness by constructing a game that permits a winning $k$-history dependent strategy and forces all other $(k-1)$-history dependent strategies to lose. Intuitively, we can construct a game in which the environment reveals a bit, the play then traverses a corridor of $k-1$ indistinguishable cells, and the property---obtained from a window predicate via Lemma~\ref{lem:windows}---requires the controller to reproduce the bit at the end of the corridor. At decision time, the bit lies exactly $k$ steps in the past, so a $k$-history dependent strategy wins, while every $(k-1)$-history dependent strategy observes only the corridor, confuses the two reachable bit prefixes, and loses on one of them. 

\begin{restatable}{theorem}{hardness}
\label{thrm:hardness}
    For all $k\geq 2$, there exists a differential safety property $\varphi$ of order $k$ and a differential safety game $\gG_\varphi$, with $k$-history dependent winning strategies, but no $(k-1)$-history dependent winning strategies.
\end{restatable}


\section{Iterative Shield Synthesis for Safety Hierarchies}
\label{sec:hierarchical-safety}
So far, we considered a single differential safety property $\varphi$ of order
$k$. In many applications, however, safety is not specified by one isolated
constraint, but by a hierarchy of increasingly informative safety conditions:
every level constrains one additional derivative while preserving the
requirements of the levels below.

\begin{example}[Ex.~\ref{ex:car-game} cont.]
\label{ex:hierarchical-motivation}
For the car of Example~\ref{ex:car2}, safety naturally decomposes into
constraints of increasing order: the car must avoid the walls, respect a speed
limit, and avoid excessive acceleration, jerk, and higher-order changes. This induces a hierarchy
where each level preserves all lower-order requirements, i.e., 
\begin{align*}
\varphi_1
\equiv
\partial^0\in\aS\setminus\{a,b\},\; 
\varphi_2
\equiv
\varphi_1\land \partial^1\in[v_{\min},v_{\max}],\; 
\varphi_3
\equiv
\varphi_2\land \partial^2\in[a_{\min},a_{\max}],\; 
\dots
\end{align*}

\end{example}


\subsection{Hierarchical safety conditions.}
We formalise this idea by a sequence $\varphi_1,\dots,\varphi_k$ of
differential safety properties with $\ord(\varphi_i)= i-1$, so that
$\varphi_i$ is evaluated on windows of $i$ consecutive states.

\begin{definition}[Differential safety hierarchy]
\label{def:hierarchy}
A sequence $\varphi_1,\dots,\varphi_k$ is a \emph{hierarchical differential
safety condition} if, for every $i\in\{2,\dots,k\}$ and every window
$(s_1,\dots,s_i)\in \aS^i$,
\begin{align}
    \label{eq:hierarchy}
     \sem{\varphi_i}_{i}(s_1,\dots,s_i)=\top
    \implies
    \sem{\varphi_{i-1}}_{i-1}(s_2,\dots,s_i)=\top .
\end{align}
\end{definition}
Intuitively, a window that is safe at level $i$ remains safe at level $i-1$
after forgetting its oldest state. In particular, every conjunction
$\varphi_i\coloneqq \psi_1\land\dots\land\psi_i$ of differential safety
properties $\psi_j$ with $\ord(\psi_j)= j-1$ is a hierarchical condition,
since the atoms of $\psi_1,\dots,\psi_{i-1}$ depend only on the last $i-1$
states of the window.

The winning condition associated with a hierarchical differential safety condition enforces every level from its own start time,
\begin{align*}
    \aF_{\varphi_1,\dots,\varphi_k}
    \coloneqq
    \{w\in\aS^\omega \mid \forall t\in\pNN\colon \sem{\varphi_{\min\{t,k\}}}_t(w)=\top\}.
\end{align*}

\paragraph{Winning windows.}
For each level $i$, let $\gG_{\varphi_i}$ denote the differential safety game
with winning condition $\varphi_i$. We call length-$i$ histories \emph{windows} to emphasise the sliding evaluation of the level-$i$ property. We call a window
$(s_1,\dots,s_i)\in\aS^i$ \emph{winning at level $i$} if the controller has a
strategy that always permits at least one action and guarantees
$\sem{\varphi_i}_t(w)=\top$ for all $t\geq i$ along every consistent infinite
extension $w=s_1\cdots s_i s_{i+1}\cdots$. We write $\aW_i\subseteq\aS^i$ for
the set of winning windows and define the safe action sets
\begin{align*}
    \xi_i(s_1,\dots,s_i)
    \coloneqq
    \{a\in\aA \mid \forall s'\in\dT(s_i,a)\colon (s_2,\dots,s_i,s')\in\aW_i\}.
\end{align*}
These are the windowed counterparts of the objects computed in
Section~\ref{sec:synthesis}: on every history whose current length-$i$ window
lies in $\aW_i$, the maximally permissive winning strategy of $\gG_{\varphi_i}$
permits precisely the actions $\xi_i$ of that window
(cf.\ Lemma~\ref{lem:naive}). Moreover, for $i\geq 2$ the set
$\xi_i(s_1,\dots,s_i)$ does not depend on $s_1$, so the induced shield is
$(i-1)$-history dependent, in line with Theorem~\ref{thrm:algo1}.

\subsection{Shield Synthesis}



There are two natural approaches to finding the maximally permissive winning
strategy for a hierarchical differential safety condition. The first approach is \emph{direct synthesis}: apply
Algorithm~\ref{algo:1} to the top-level property $\varphi_k$. By Equation~\eqref{eq:hierarchy}, whenever $\varphi_k$
holds at time $t$, so do $\varphi_{k-1},\dots,\varphi_1$, hence the level-$k$
shield enforces the entire hierarchy from time $k$ onwards; the first $k-1$
steps of a run require the separate transient pass described below. However, direct
synthesis may be wasteful, as it materialises the full window space $\aS^k$.
The second approach is \emph{iterative synthesis}: solve the hierarchy level by
level, using the winning windows and safe actions of level $i-1$ to prune the
construction at level $i$. The key observation is that winning is closed under
taking suffixes. 
Intuitively, a winning window remains winning one level below, and every action
that is safe at level $i$ is also safe at level $i-1$.

\begin{restatable}{lemma}{hierarchicalSuffix}
\label{lem:hierarchical-suffix}
Let $\varphi_1,\dots,\varphi_k$ be a hierarchical differential safety
condition. For every $i\in\{2,\dots,k\}$ and every $(s_1,\dots,s_i)\in\aS^i$
\begin{align*}
    (s_1,\dots,s_i)\in\aW_i
    \implies
    (s_2,\dots,s_i)\in\aW_{i-1}
    \quad\text{and}\quad
    \xi_i(s_1,\dots,s_i)\subseteq\xi_{i-1}(s_2,\dots,s_i).
\end{align*}
\end{restatable}

\paragraph{Iterative synthesis.}
Lemma~\ref{lem:hierarchical-suffix} yields an exact pruning
rule: a window can only be winning at level $i$ if its suffix is winning at
level $i-1$, and an action can only be safe if it was already safe one level
below. Hence, at level $i$ it suffices to materialise the candidate windows
\begin{align*}
    \aC_i
    \coloneqq
    \bigl\{
        (s_1,\dots,s_i)\in\aS^i
        \mid
        \sem{\varphi_i}_{i}(s_1,\dots,s_i)=\top
        \text{ and }
        (s_2,\dots,s_i)\in\aW_{i-1}
    \bigr\},
\end{align*}
and, at each candidate, only the inherited actions
$\xi_{i-1}(s_2,\dots,s_i)$.

This is implemented in Algorithm~\ref{algo:hierarchical-synthesis}, 
which solves  each level in the same manner as Algorithm~\ref{algo:1}, but on the pruned window space.
At level $i$, it first constructs the candidate set $\aC_i$
(Line~\ref{alg2:candidates}); the conventions $\aW_0\coloneqq\{()\}$ and
$\xi_0(())\coloneqq\aA$ make level $1$ an instance of the general loop. An
inherited action is initially considered safe if every successor window is
again a candidate (Lines~\ref{alg2:inherited}--\ref{alg2:add-safe}); note that
successor windows automatically have winning suffixes by definition of
$\xi_{i-1}$, so this test reduces to evaluating $\varphi_i$. Candidates without
safe actions are inserted into the conflict set $\aU$
(Line~\ref{alg2:empty}). Conflicts are then propagated backwards exactly as in
\textsc{PropagateConflicts}, where all windows have full length $i$ and the
$a$-predecessors of a window $\bar s=(s_1,\dots,s_i)$ are obtained by shifting
one step:
\begin{align*}
    \Predop^a(\bar s)
    \coloneqq
    \{\bar s'\in\aS^i \mid \bar s'_{2:i}=\bar s_{1:i-1}
    \text{ and } \bar s_i\in\dT(\bar s'_i,a)\}.
\end{align*}
Whenever a window is unsafe, the actions leading to it are removed from its
candidate predecessors, and predecessors left without safe actions are inserted
into $\aU$ in turn (Lines~\ref{alg2:prop-start}--\ref{alg2:remove}). The
windows that retain a safe action are exactly the winning windows of level $i$
(Line~\ref{alg2:winning}), which seed the construction of level $i+1$. A final
transient pass (Lines~\ref{alg2:transient-start}--\ref{alg2:transient-end})
computes the safe actions during the first $k-1$ steps of a run; it is
explained below.

\begin{example}[Ex.~\ref{ex:hierarchical-motivation} cont.]
\label{ex:hierarchical-car-iterative}
Iterative synthesis first computes the
positions from which the walls can be avoided. At level $2$, it considers only
velocity windows $(s_1,s_2)$ whose suffix $s_2$ is position-winning. At level
$3$, it considers only acceleration windows $(s_1,s_2,s_3)$ whose suffix
$(s_2,s_3)$ is velocity-winning. Hence windows ending too close to a wall, or
moving too fast towards it, are discarded before higher-order constraints are
explored. 
\end{example}

\begin{algorithm}[t]
\begin{algorithmic}[1]
\caption{Iterative synthesis algorithm}
\label{algo:hierarchical-synthesis}
\Procedure{IterativeSynthesis}{$\gG,\varphi_1,\dots,\varphi_k$}
    \State \textbf{Initialize} $\aW_0\gets\{()\}$, $\xi_0(())\gets\aA$ \Comment{the empty window carries no constraint}
    \For{$i=1,\dots,k$}
        \State $\aU\gets\emptyset$ and $\xi_i(\bar s)\gets\emptyset$ for all $\bar s\in\aS^i$ \Comment{stored lazily with default $\emptyset$}
        \State $\aC_i\gets\{(s_1,\dots,s_i)\in\aS^i \mid \sem{\varphi_i}_{i}(s_1,\dots,s_i)=\top \wedge (s_2,\dots,s_i)\in\aW_{i-1}\}$ \Comment{prune by level $i-1$} \algline{alg2:candidates}
        \For{$(s_1,\dots,s_i)\in\aC_i$}
            \For{$a\in\{a\in\xi_{i-1}(s_2,\dots,s_i)\mid \forall s'\in\dT(s_i,a)\colon (s_2,\dots,s_i,s')\in\aC_i\}$} \algline{alg2:inherited}
                \State $\xi_i(s_1,\dots,s_i)\gets\xi_i(s_1,\dots,s_i)\cup\{a\}$ \Comment{inherited action} \algline{alg2:add-safe}
            \EndFor
            \If{$\xi_i(s_1,\dots,s_i)=\emptyset$} \algline{alg2:empty}
                \State $\aU\gets\aU\cup\{(s_1,\dots,s_i)\}$ \Comment{no safe action remains}
            \EndIf
        \EndFor
        \While{$\aU\neq\emptyset$} \Comment{as \textsc{PropagateConflicts}, on full-length windows} \algline{alg2:prop-start}
            \State $\bar s\gets\mathrm{Pop}(\aU)$
            \For{$a\in\aA$ and $\bar s'\in\Predop^a(\bar s)\cap\aC_i$} \algline{alg2:pred}
                \If{$\xi_i(\bar s')=\{a\}$}
                    \State $\aU\gets\aU\cup\{\bar s'\}$ \Comment{removing $a$ makes $\bar s'$ unsafe}
                \EndIf
                \State $\xi_i(\bar s')\gets\xi_i(\bar s')\setminus\{a\}$ \Comment{avoid forcing an unsafe window} \algline{alg2:remove}
            \EndFor
        \EndWhile
        \State $\aW_i\gets\{\bar s\in\aC_i \mid \xi_i(\bar s)\neq\emptyset\}$ \Comment{winning windows of level $i$} \algline{alg2:winning}
    \EndFor
    \State $\hat\xi_k\gets\xi_k$ \Comment{transient pass: run initialisation} \algline{alg2:transient-start}
    \For{$j=k-1,\dots,1$ \textbf{and} $(s_1,\dots,s_j)\in\aS^j$}
        
\State {\scriptsize $\hat\xi_j(s_1 \dots s_j)\gets\{a\in\aA \mid \sem{\varphi_j}_{j}(s_1\dots s_j)=\top \wedge \forall s'\in\dT(s_j,a)\colon \hat\xi_{j+1}(s_1 \dots s_j,s')\neq\emptyset\}$
} \algline{alg2:transient-end}
    \EndFor
    \State \Return $(\aW_k,\xi_k,\hat\xi_{k-1},\dots,\hat\xi_1)$
\EndProcedure
\end{algorithmic}
\end{algorithm}


\begin{restatable}[Exactness of iterative synthesis]{theorem}{hierarchicalIterative}
\label{thm:hierarchical-iterative}
Let $\varphi_1,\dots,\varphi_k$ be a hierarchical differential safety
condition. For every $i\in\{1,\dots,k\}$,
Algorithm~\ref{algo:hierarchical-synthesis} computes $\aW_i$ and the
restriction of $\xi_i$ to $\aW_i$. In particular, iterative synthesis yields
the same maximally permissive winning strategy for $\gG_{\varphi_k}$ as direct
synthesis, if it exists. Moreover, the strategy $\sigma$ defined by
$\sigma(w)\coloneqq\hat\xi_{|w|}(w)$ for $|w|<k$ and
$\sigma(w)\coloneqq\xi_k(w_{|w|-k+1:|w|})$ for $|w|\geq k$ is the maximally
permissive winning strategy for the turn-based game with the hierarchical
winning condition $\aF_{\varphi_1,\dots,\varphi_k}$. The time complexity is
$\mathcal O\bigl(\sum_{i=1}^{k}|\aC_i|\cdot|\aA|\cdot|\aS|\bigr)
\subseteq \mathcal O(|\aS|^{k+1}\cdot|\aA|)$.
\end{restatable}
In the worst case every lower-level window is winning, $\aC_i=\aS^i$, and the
complexity matches that of direct synthesis (Theorem~\ref{thrm:algo1}), with some overhead introduced by the iterative logic.
In practical cases, however, we expect the lower levels to already remove large parts of the window space, leading to a practical efficiency gain. We demonstrate this in our experimental evaluation.

\section{Experimental Evaluation}
\label{sec:experiments}

We evaluate the synthesis procedures on vehicle models on a 2D discretized grid, using an explicit-state prototype implementation.
The goal of our evaluation is to demonstrate the feasibility of our approach, the types of traces produced by differential safety properties with increasing orders, and to investigate how different synthesis methods compare in terms of time and memory usage.

\subsection{Experimental Setup}



\paragraph{Two-dimensional car game.}
Our experiments instantiate the differential safety-game construction on a
discretized two-dimensional car model, inspired by a vehicle that moves on a
bounded grid while approaching a wall or an obstacle.  The state space is
  $\aS = \{0,\ldots,X\}\times \{0,\ldots,Y\}\subseteq \mathbb{Z}^2$,
where a state \(q_t=(x_t,y_t)\) is the position of the vehicle at time \(t\).
At each step the controller chooses a commanded displacement
$  u=(u_x,u_y)\in \aA :=
  \{0,\ldots,V_x\}\times\{0,\ldots,V_y\}$.
The environment then chooses an additive disturbance
\[
  \delta=(\delta_x,\delta_y)\in
  [-\eta_x(u),\eta_x(u)]\times[-\eta_y(u),\eta_y(u)]\cap\mathbb{Z}^2,
  \qquad
  \eta_i(u)=\min\{\eta_{\max},u_i\},
\]
and the next state is \(q_{t+1}=q_t+u+\delta\) (clipped to the state space).  
Actions represent control commands, while nondeterminism models uncertainty.
The game has an absorbing end state at the top-right corner.
Positional unsafety is defined by a wall \(x>x_{\mathsf{wall}}\) that must not be crossed.
higher-order safety constraints are given by bounds on the magnitude of each derivative, 
with the timestep normalized to $\Delta \tau =1$.
The $k$-th order constraint is represented as $\|\partial^k\|_2\leq r_k$, for some radius $r_k$.
Figure~\ref{fig:2d-car-game} illustrates the setup.
Each instance is parametrized by the space and speed bounds $(X,Y)$, $(V_x, V_y)$, the maximum disturbance $\eta_{\max}$, and the sequence of radii for the higher-order properties $[r_1,\dots, r_k]$.
We will use the notation $E^{(X,Y), (V_x, V_y), \eta_{\max}}_{[r_1,\dots, r_k]}$ to denote an instance, with the corresponding parameter values.
All experiments were performed with a Macbook (M2)\footnote{The source code to reproduce the experimental evaluation is available at \url{https://github.com/filipcano/differential-safety-shields}.}.

\begin{figure}[t]
  \centering
  \includegraphics[width=0.5\linewidth]{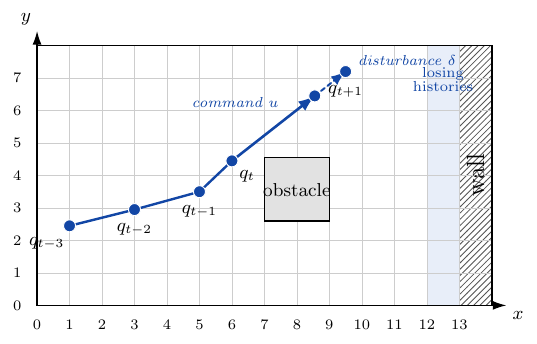}
  \caption{Two-dimensional car game used in the experiments.  
  }
  \label{fig:2d-car-game}
\end{figure}

\subsection{Synthesis cost}
We build a benchmark with a diverse set of parameters, and report the computing time and maximum memory usage of each synthesis procedure in Table~\ref{tab:synth-cost-main}. 
The timeout is $60s$. We did not include a memory limit, as no instance used too much memory before the timeout.
We also report the fraction of initial states that the iterative algorithm prunes at the beginning of each iteration. In general, both optimized solutions -- direct and iterative -- heavily outperform the basic method. 
Furthermore, we observe that the iterative method outperforms the direct one when the number of pruned states is large. This is consistent with our intuition.

\begin{table*}
\centering
\setlength{\tabcolsep}{0.8\tabcolsep}
\caption{Synthesis Cost: Time and Memory}
\label{tab:synth-cost-main}
\begin{tabular}{lccc|ccc|rrr}
\toprule
\multirow{2}{*}{Benchmark}
& \multicolumn{3}{c|}{Synth. Time (s)}
& \multicolumn{3}{c|}{Synth. Memory (MB)}
& \multicolumn{3}{c}{Pruned States (\%)}\\
\cline{2-10}
& Baseline & Direct & Iterative
& Baseline & Direct & Iterative
& $d = 3$ & $d = 4$ & $d = 5$\\
\midrule
$E^{(40, 40), (4, 4), 1}_{[4, 5, 7]}$
& 39.7 & 1.92 & \textbf{1.37}
& 2420.2 & 220.1 & \textbf{219.2}
& 0.0\% & - & - \\

$E^{(60, 60), (8, 4), 1}_{[4, 5, 7]}$
& TO & \textbf{3.7} & 3.9
& TO & \textbf{550.3} & 565.9
& 0.0\% & - & - \\

$E^{(470, 120), (12, 6), 1}_{[5, 2, 8, 10]}$
& TO & 48.1 & \textbf{19.1}
& TO & 2465.1 & \textbf{956.4}
& 89.5\% & 89.7\% & - \\

$E^{(540, 120), (12, 6), 1}_{[5, 2, 8, 10]}$
& TO & TO & \textbf{24.3}
& TO & TO & \textbf{883.4}
& 89.5\% & 89.7\% & - \\

$E^{(540, 120), (16, 6), 1}_{[6, 3, 9, 10]}$
& TO & TO & TO
& TO & TO & TO
& 0.0\% & 0.0\% & - \\

$E^{(20, 20), (4, 2), 1}_{[4, 4, 6, 6, 20]}$
& 52.8 & 8.1 & \textbf{2.54}
& 4661.3 & 788.8 & \textbf{131.5}
& 0.0\% & 56.8\% & 90.0\% \\

$E^{(20, 20), (4, 2), 1}_{[3, 3, 6, 10, 20]}$
& 46.5 & \textbf{5.44} & 6.52
& 4517.5 & \textbf{607.3} & 613.9
& 0.0\% & 0.0\% & 58.7\% \\

$E^{(30, 20), (4, 2), 1}_{[3, 3, 4, 4, 20]}$
& 4.27 & 1.17 & \textbf{0.295}
& 378.3 & 114.2 & \textbf{28.5}
& 0.0\% & 66.8\% & 92.9\% \\

$E^{(80, 20), (6, 4), 1}_{[5, 2, 8, 10, 12]}$
& 20.7 & 2.82 & \textbf{1.08}
& 1229.1 & 206.1 & \textbf{68.8}
& 88.6\% & 88.7\% & 88.3\% \\

$E^{(120, 20), (6, 2), 1}_{[5, 2, 8, 10, 12]}$
& 27.2 & 3.29 & \textbf{1.34}
& 1776.1 & 279.4 & \textbf{105.4}
& 85.2\% & 85.4\% & 84.9\% \\
\bottomrule
\end{tabular}
\end{table*}

\subsection{Qualitative Trace Evaluation}
To demonstrate how the shields work, we synthesize shields for increasing orders and simulate several traces of a shielded agent on two instances: \texttt{Wall}, parametrized as $E^{(100, 18), (6,6), 2}_{[8,4,6,9]}$ and \texttt{Obstacle}, parametrized as $E^{(110, 24), (10,7), 0}_{[10,4,3,3]}$, with an obstacle centred at position $(65,4)$. 
For the simulation, we use an agent policy that favours moving along the $x$-axis and maintaining a target speed, and implement the safety as pre-shields: the agent can only choose among actions allowed by the shield. As it approaches either the obstacle or the wall, the agent is forced to either brake or steer away from it, and typically favours steering, as it can maintain the target speed. Further details on how the agent's policy works are in Appendix~\ref{app:2d-car-model}.
Fig.~\ref{fig:traces} illustrates the instances and simulated traces of increasing orders.
In both cases, observe that as the order of the shield increases, the traces tend to stay away from large speeds and sharp turns.

\paragraph{Derivative pressure.}
To visualize how different shields constrain the agent differently, we introduce the concept of \emph{derivative pressure}. 
For a derivative constraint of order $d$ with radius $r_d$, we define the
\emph{pressure} at time $t$ as
  $\rho_d(t) = \|\partial^d x_t\|/r_d$.
Thus $\rho_d(t) < 1$ means that the trajectory is strictly inside the
$d$-th order safety bound, $\rho_d(t)=1$ means that it lies on the boundary,
and $\rho_d(t)>1$ indicates a violation. This normalization makes constraints
of different orders directly comparable: values above one measure both the occurrence of a violation and its relative magnitude.
In Fig.~\ref{fig:pressure-main}, we show, for increasing orders of the constraint, the derivative pressure corresponding with each shield, for the \texttt{Wall} instance.
Observe that shields for a higher order than the constraint always stay below the bound, and in general, higher order shields tend to be more conservative in all constraints orders.
Fig.~\ref{fig:pressure-main} depicts that agents with lower order shields tend to end their traces sooner, as they go faster.


\begin{figure}
     \centering
     \begin{subfigure}[b]{0.48\textwidth}
         \centering
         \includegraphics[width=\textwidth]{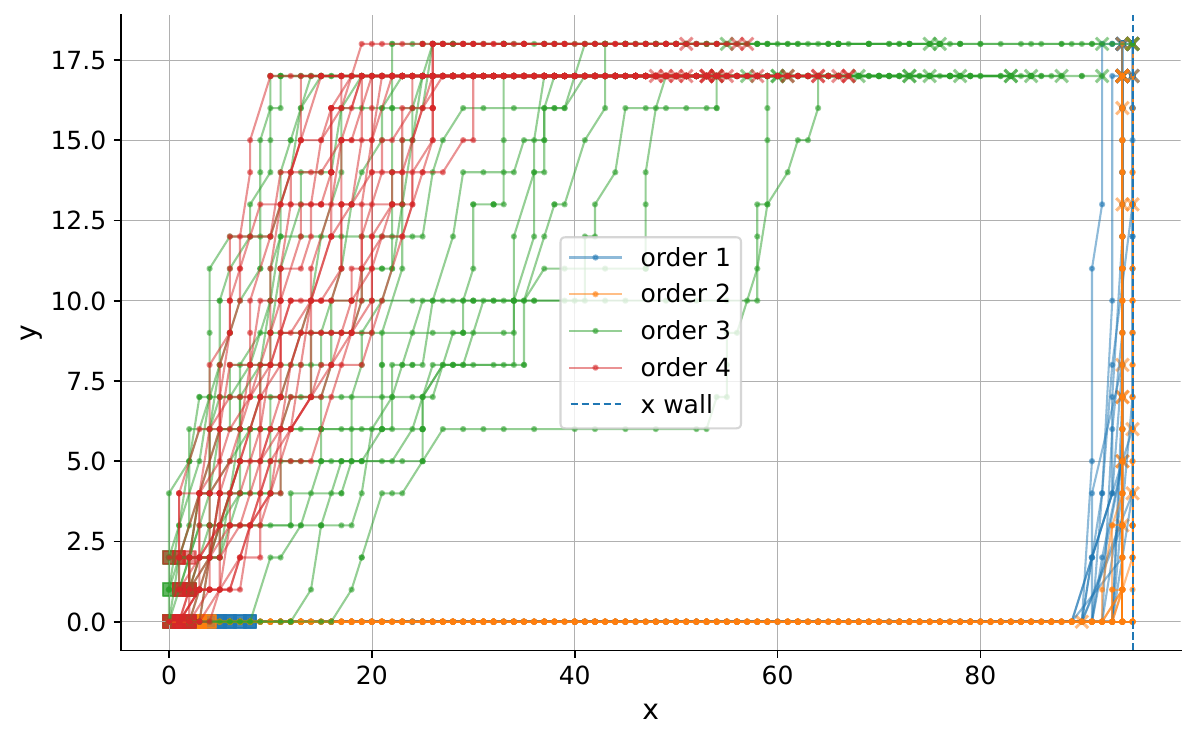}
     \end{subfigure}
     \hfill
     \begin{subfigure}[b]{0.48\textwidth}
         \centering
         \includegraphics[width=\textwidth]{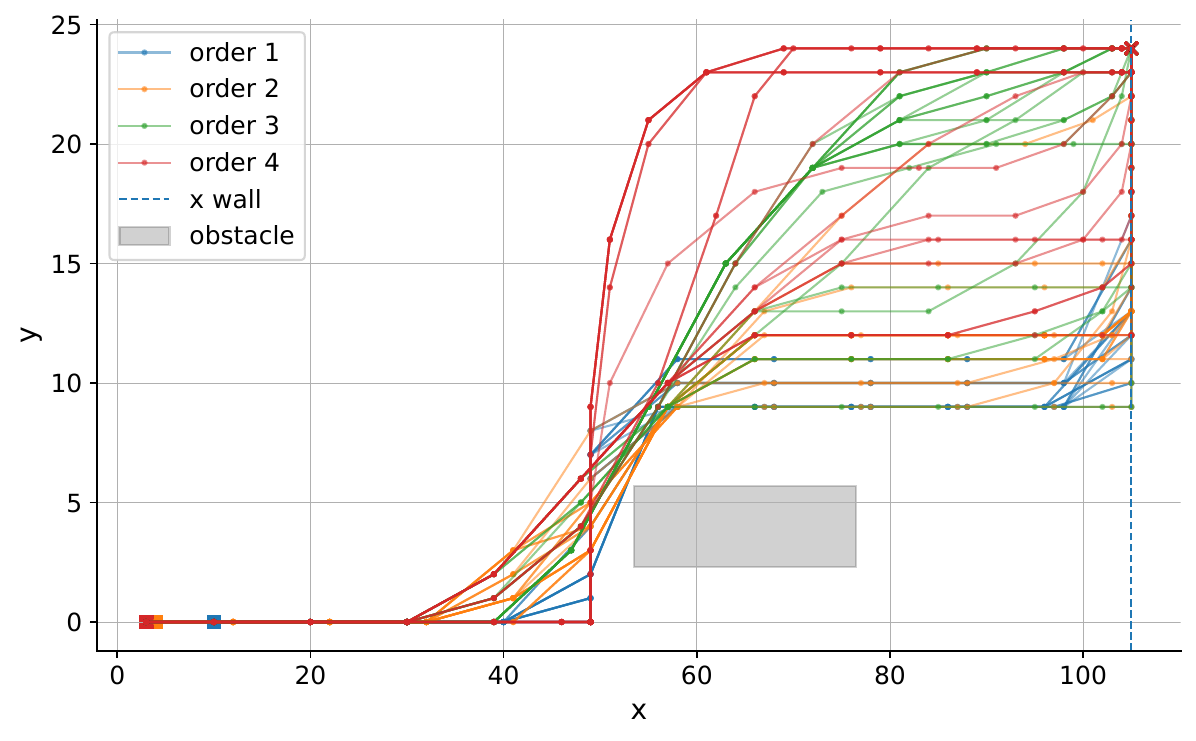}
     \end{subfigure}
        \caption{Trace demonstration}
        \label{fig:traces}
\end{figure}

\begin{figure}[t]
    \centering
    \includegraphics[width=0.92\linewidth]{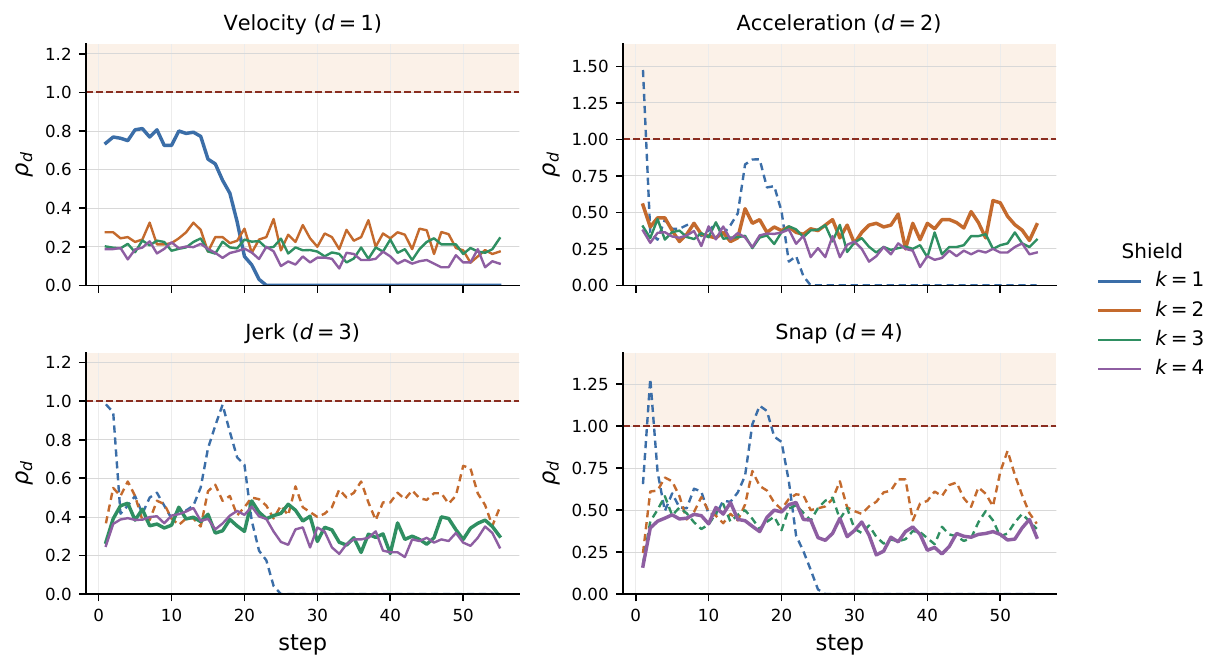}
    \caption{Derivative pressure by shield and constraint order for the \texttt{Wall} instance.}
    \label{fig:pressure-main}
\end{figure}




\section{Related Work}
\label{sec:related-work}

\paragraph{Shielding, runtime enforcement, and games.}
Shield synthesis enforces safety at runtime by computing the winning region of a
safety game and blocking, or correcting, actions that may leave it
\cite{bloem2015shield}. Shields have been used to make reinforcement learning
safe during training and deployment \cite{alshiekh2018safe}, and have been
extended to probabilistic \cite{Jansen2020ProbabilisticShields} and partially
observable settings \cite{Carr2023PartialObservability}. More broadly, shields
are reactive runtime-enforcement mechanisms: unlike general edit automata
\cite{Schneider2000Enforceable,LigattiBauerWalker2005EditAutomata,Falcone2011RuntimeEnforcement},
they must react online and interfere only when necessary. The underlying notions
of maximally permissive strategies and safety-game fixpoints are classical
\cite{RamadgeWonham1987Supervisory,Bernet2002Permissive,thomas1995synthesis,GradelThomasWilke2002Games}.
Closest to our construction are approaches that augment the game state with
finite memory, for instance to handle delayed interaction or delayed observation
\cite{chen2018delayedInteraction,cano2023delayedShielding}. Our memory has a
different source: it is the finite window required to evaluate discrete
derivatives. Differential safety properties are $\omega$-regular and could
therefore be handled by generic automata-theoretic constructions, but our
finite-difference structure yields an explicit memory bound, a direct history-game
construction, and an exact level-by-level pruning algorithm.

\paragraph{Numeric and cyber-physical shields.}
Several works extend shielding beyond Boolean transition systems. Real-valued
shields synthesize enforcers for cyber-physical systems with real-valued signals
by combining compatibility analysis with safety-game solving
\cite{Wu2019ShieldSynthesisReal}, while shield synthesis for LTL modulo theories
allows temporal specifications over theory atoms
\cite{Rodriguez2025ShieldLTLModuloTheories}. These approaches also address
data-rich behaviours. In contrast, differential safety properties deliberately
focus on backward finite differences over a finite abstraction. This restriction fixes the required memory, which enables pruning across position,
velocity, acceleration, jerk, and higher-order constraints.

\paragraph{Safety filters, barrier functions, and trajectory generation.}
In continuous control, derivative constraints are commonly enforced by control
barrier functions and safety filters
\cite{Ames2019ControlBarrier,Hsu2024SafetyFilter}, including predictive safety
filters \cite{WabersichZeilinger2021PredictiveSafetyFilter} and reachability-based
safe learning methods
\cite{Fisac2019GeneralSafetyFramework,Akametalu2014ReachabilitySafeLearning}.
Particularly relevant are high-relative-degree and high-order control barrier
functions, which handle constraints whose satisfaction depends on derivatives of
the system output
\cite{NguyenSreenath2016ExponentialCBF,XiaoBelta2022HighOrderCBF}. Derivative
bounds also arise in online trajectory generation, including jerk-limited motion
planning \cite{KrogerWahl2010OTG,BerscheidKroger2021Ruckig}. These methods
typically rely on a continuous dynamics model and often solve online
optimisation problems. Our approach instead solves a finite adversarial game
offline and obtains worst-case, maximally permissive shields for the chosen
abstraction.

\paragraph{Specification languages and window objectives.}
Signal temporal logic constrains real-valued signals over continuous time and
supports quantitative monitoring and falsification
\cite{MalerNickovic2004STL,FainekosPappas2009Robustness,DonzeMaler2010RobustSatisfaction,Deshmukh2017RobustOnlineSTL}.
Differential safety properties can be viewed as a restricted safety fragment over
backward finite differences: less expressive than STL, but tailored to
bounded-memory enforcement. The use of sliding windows also connects to window
objectives in games, such as window mean-payoff, total-payoff, and timed window
objectives \cite{chatterjee2015looking,MainRandourSproston2021TimedWindowParity}.
There, windows provide bounded-response quantitative guarantees; here, the
window is the finite history needed to evaluate derivatives. This difference fixes
the exact memory required by the shield and exposes the closure property required for iterative synthesis.

\section{Conclusion}
\label{sec:conclusion}
We introduced differential safety properties, which specify constraints on discrete derivatives, and analyse their expressiveness. 
We propose two shield synthesis algorithms: a \emph{direct} memory optimal algorithm and an iterative  algorithm for hierarchical differential safety properties which permits successive pruning. We implement higher-order safety shields to demonstrate the feasibility of our approach and the performance improvement of the proposed synthesis algorithms with respect to the baseline.

Several directions remain open. 
First, symbolic fixed point methods could be used to improve upon the explicit state enumeration present in our synthesis algorithms.
Second, quantitative shields could trade interference against safety margins on the derivatives.
Third, the impact of the discrete shield on the underlying continuous dynamics should be subject of study. 
Fourth, differential constraints could be combined with other sliding-window specifications, such as window payoff objectives \cite{chatterjee2015looking}.

\subsubsection{\ackname} This work was supported by the European Research Council under Grant No.: ERC-2020-AdG 101020093. This work was also supported by the ISTA Responsible AI Program, made possible through the support of Garrett Camp and the Camp Foundation.



\bibliographystyle{splncs04}
\bibliography{references}

\appendix

\section{Definitions}
\label{sec:app:definitions}

\paragraph{Shields.}
There are two types of shields: a pre-shield that disables precisely those actions that are not permitted, and a post-shield that may either keep the proposed action unchanged or replace it by other actions.
A pre-shield is a function $\sigma_{\mathrm{pre}}\colon \aS^*\to 2^{\aA}$ that restricts the actions available to the controller before an action is chosen. Given a controller strategy $\xi$, the corresponding pre-shielded strategy $(\xi \circ \sigma_{\mathrm{pre}})\colon \aS^*\to 2^{\aA}$ is defined for every finite sequence $w\in \aS^*$ as
\begin{align*}
( \xi \circ \sigma_{\mathrm{pre}})(w)
\coloneqq
\xi(w)\cap \sigma_{\mathrm{pre}}(w).
\end{align*}
A post-shield is a function $\sigma_{\mathrm{post}}\colon \aS^*\times \aA\to 2^{\aA}$ that observes the action proposed by the controller and returns the set of actions that may be executed instead. Given a controller strategy $\xi$, the corresponding post-shielded strategy $(\sigma_{\mathrm{post}} \circ \xi)\colon \aS^*\to 2^{\aA}$ is defined for every finite sequence $w\in \aS^*$ as
\begin{align*}
( \sigma_{\mathrm{post}} \circ \xi)(w)
\coloneqq
\bigcup_{a\in \xi(w)} \sigma_{\mathrm{post}}(w,a).
\end{align*}
A shield is sound if the induced shielded strategy is winning. A sound shield is minimally interfering if it disables or replaces an action only when executing that action would make it impossible to guarantee the winning condition. It is well known that, whenever a maximally permissive winning strategy exists, it implements minimally interfering shields: the shield permits precisely the actions allowed by the maximally permissive winning strategy~\cite{bloem2015shield}. In particular, a maximally permissive winning strategy $\xi$ directly serves as a pre-shield, $\sigma_{\mathrm{pre}}\coloneqq\xi$, and induces a canonical post-shield that keeps the proposed action whenever it is permitted and otherwise offers the permitted alternatives,
\begin{align*}
\sigma_{\mathrm{post}}(w,a)
\coloneqq
\begin{cases}
\{a\} & \text{if } a\in\xi(w),\\
\xi(w) & \text{otherwise.}
\end{cases}
\end{align*}
Hence, from the point of view of synthesis, pre- and post- shields are the same, as both are derived from the maximally permissive winning strategy. In the experiments we use pre shields.

\section{Remarks}
\label{sec:app:remarks}

\paragraph{Scaling and the role of $\Delta\tau$.}
The finite-difference definition divides by $(\Delta\tau)^k$. If the safety predicates introduced below are evaluated in an integer-valued constraint language, these divisions can be eliminated by rescaling: measuring time in units of one sampling period yields $\Delta\tau=1$, and a physical bound on an order-$k$ derivative translates by multiplying with $(\Delta\tau)^k$, e.g.,
\begin{align*}
  L \leq \partial_t^k(w) \leq U
  \quad\text{becomes}\quad
  (\Delta\tau)^k L \leq \sum_{i=0}^{k}(-1)^i\binom{k}{i}s_{t-i} \leq (\Delta\tau)^k U.
\end{align*}
If positions are stored in grid units with physical spacing $\eta$, every order-$k$ derivative is additionally scaled by $\eta$.

\paragraph{Representation and complexity.}
Since $\lnot(\partial^k\in\aR)\equiv\partial^k\in\RN^d\setminus\aR$, negations can be pushed into the atoms, so atoms may be assumed positive. For the complexity statements below, we assume that the safe sets $\aR$ are represented such that membership is decided in constant time, e.g., as intervals or finite unions of boxes; the algorithms then evaluate $\sem{\varphi}_t$ in time $\mathcal O(|\varphi|)$ per window, which we treat as constant.

\paragraph{Lipschitz constant.}
Suppose every $t\geq 2$ of a sequence $w$ satisfies $\partial^1_t(w)\in\{v\in\RN^d \mid \|v\|\leq v_{\max}\}$. Then, by the triangle inequality, $\|s_t-s_{t'}\|\leq v_{\max}\cdot\Delta\tau\cdot|t-t'|$ for all indices $t,t'$, i.e., $w$ is the sampling of a $v_{\max}$-Lipschitz curve. A shield enforcing a first-order bound therefore bounds the ``time'' Lipschitz constant of every shielded controller, regardless of the controller's internal complexity, a prerequisite for many verification techniques for systems with learned components~\cite{huang2019reachnn,fan2020reachnnstar,zhang2023reachability}.

\paragraph{Initial steps.}
The global semantics of an order-$k$ property constrains a sequence only from time $k+1$ onwards. For conjunctions of atoms of different orders, this under-constrains the first steps of a sequence: in Example~\ref{ex:car2} below, position safety is not required at times $1$ and $2$, even though $\partial^0\in\aS\setminus\{a,b\}$ is a conjunct that could already be evaluated there.

\section{Proofs}

\paragraph{Notation.}
By Lemma~\ref{lem:windows}(2), for every differential safety property
$\varphi$ of order $k$, the value $\sem{\varphi}_t(w)$ depends only on the
window $(w_{t-k},\dots,w_t)$; moreover, since the derivatives
$\partial^0_t,\dots,\partial^k_t$ are the same finite differences of the last
$k+1$ states at every index, the value is independent of the position of the
window within the sequence,
\begin{align}
    \label{eq:translation-invariance}
    \sem{\varphi}_t(w)=\sem{\varphi}_{k+1}(w_{t-k},\dots,w_t)
    \qquad\text{for all } k+1\leq t\leq|w|.
\end{align}
In particular, for the hierarchical properties of
Section~\ref{sec:hierarchical-safety} with $\ord(\varphi_i)=i-1$, we have
$\sem{\varphi_i}_t(w)=\sem{\varphi_i}_{i}(w_{t-i+1},\dots,w_t)$ for all
$t\geq i$. We use this translation invariance throughout.

\subsection{Proofs of Section~\ref{sec:dsp}}

\derivativeClosedForm*

\begin{proof}
By induction on $k$. For $k=0$, the formula yields $\partial_t^0(w)=s_t$ for
all $0<t\leq|w|$, matching the definition. For the inductive step, let
$k\geq 1$ and $k<t\leq|w|$. Both $\partial_t^{k-1}(w)$ and
$\partial_{t-1}^{k-1}(w)$ are defined, since $k-1<t-1$ and $t-1\leq|w|$.
Applying the induction hypothesis to both and shifting the summation index of
the second sum by one,
\begin{align*}
\partial_t^k(w)
&=\frac{1}{\Delta\tau}\Bigl(\partial_t^{k-1}(w)-\partial_{t-1}^{k-1}(w)\Bigr)\\
&=\frac{1}{(\Delta\tau)^k}\Bigl(
\sum_{i=0}^{k-1}(-1)^i\binom{k-1}{i}s_{t-i}
-\sum_{i=1}^{k}(-1)^{i-1}\binom{k-1}{i-1}s_{t-i}
\Bigr)\\
&=\frac{1}{(\Delta\tau)^k}
\sum_{i=0}^{k}(-1)^i\Bigl(\binom{k-1}{i}+\binom{k-1}{i-1}\Bigr)s_{t-i}
=\frac{1}{(\Delta\tau)^k}\sum_{i=0}^{k}(-1)^i\binom{k}{i}s_{t-i},
\end{align*}
using the conventions $\binom{k-1}{-1}=\binom{k-1}{k}=0$ and Pascal's rule.
Conversely, for $t\leq k$ the recursive definition refers to
$\partial^0_{t-k}(w)$ with $t-k\leq 0$, and for $t>|w|$ to states beyond $w$,
so $\partial^k_t(w)$ is undefined in both cases.
\end{proof}

\windowPredicates*

\begin{proof}
(2) By induction on the structure of $\varphi$: an atom $\partial^j\in\aR$
with $j\leq k$ depends, by Lemma~\ref{lem:closed-form}, only on
$s_{t-j},\dots,s_t$, which are contained in the window, and Boolean
combinations preserve this property.

(1) Consider the map $\iota\colon\aS^{k+1}\to(\RN^d)^{k+1}$ that sends a window
$\bar v=(v_0,\dots,v_k)$, read as the states $s_{t-k},\dots,s_t$, to its
derivative values $\iota(\bar v)\coloneqq(d_0(\bar v),\dots,d_k(\bar v))$,
where $d_j(\bar v)$ is the value of $\partial^j_t$ on the window. The map
$\iota$ is injective: $d_0$ determines $s_t$, and, given
$s_t,\dots,s_{t-j+1}$, the closed form of Lemma~\ref{lem:closed-form} shows
that $d_j$ determines $s_{t-j}$, since $s_{t-j}$ occurs in $d_j$ with the
nonzero coefficient $(-1)^j/(\Delta\tau)^j$. For a window $\bar v$, define
\begin{align*}
\chi_{\bar v}\coloneqq\bigwedge_{j=0}^{k}\bigl(\partial^j\in\{d_j(\bar v)\}\bigr).
\end{align*}
Then, for every $w\in\aS^\infty$ and $k+1\leq t\leq|w|$, we have
$\sem{\chi_{\bar v}}_t(w)=\top$ iff the derivative values of the window of $w$
at $t$ equal $\iota(\bar v)$, which by injectivity holds iff the window equals
$\bar v$. The property
$\varphi_P\coloneqq\lnot\bigl(\bigwedge_{\bar v\in P}\lnot\chi_{\bar v}\bigr)$
is therefore satisfied at $t$ iff the window of $w$ at $t$ lies in $P$; the
conjunction is finite since $\aS$ is finite, and $\varphi_P$ has order $k$
because it contains atoms of order $k$ and none of higher-order. (For
$P=\emptyset$, take $\varphi_P\coloneqq\partial^k\in\emptyset$.)
\end{proof}

\subsection{Proofs of Section~\ref{sec:synthesis}}

\naive*

\begin{proof}
(1) Define the \emph{good} states of the history game as
\begin{align*}
G\coloneqq\aS^{\leq k}
\,\cup\,
\{(s_1,\dots,s_{k+1})\in\aS^{k+1} \mid \sem{\varphi}_{k+1}(s_1,\dots,s_{k+1})=\top\}.
\end{align*}
Let $w=s_1s_2\cdots\in\aS^\omega$ be a path of $\gG_\varphi$ and consider its
lifted history path $(w_{\max\{1,t-k\}:t})_{t\in\pNN}$. At time $t$, the
history state has length $\min\{t,k+1\}$; for $t\leq k$ it is good by
definition, and for $t\geq k+1$ it equals the window $w_{t-k:t}$ and is good
iff $\sem{\varphi}_{k+1}(w_{t-k:t})=\sem{\varphi}_t(w)=\top$ by
Equation~\eqref{eq:translation-invariance}. Hence the lifted path remains in
$G$ forever iff $\sem{\varphi}_t(w)=\top$ for all $t\geq k+1$, i.e., iff
$w\in\aF_\varphi$. Therefore $\aF^k_\varphi$ is the state invariance condition
``always $G$'' on $\aS^{\leq k+1}$.

(2) The map $w\mapsto(w_{\max\{1,t-k\}:t})_{t\in\pNN}$ is a bijection between
the paths of $\gG_\varphi$ starting in $\sinit$ and the paths of
$\gG^k_\varphi$ starting in $(\sinit)$, and it is compatible with actions: by
definition of $\dT^k$, the history at time $t+1$ is obtained from the history
at time $t$ and the state $s_{t+1}\in\dT(s_t,a)$. A memoryless strategy
$\xi^k$ for $\gG^k_\varphi$ and its induced strategy
$\xi(w_{1:t})\coloneqq\xi^k(w_{\max\{1,t-k\}:t})$ for $\gG_\varphi$ permit the
same actions along corresponding (finite or infinite) sequences, so the
consistent paths correspond under the bijection, the nonblocking conditions
coincide, and, by~(1), a consistent path of one game is winning iff the
corresponding path of the other is. Hence $\xi^k$ is winning for
$\gG^k_\varphi$ iff $\xi$ is winning for $\gG_\varphi$.
\end{proof}

\AlgorithmOne*

\begin{proof}
Throughout, write $\xi^{\mathrm{alg}}$ for the action sets returned by
Algorithm~\ref{algo:1}, and recall that a history $(s_1,\dots,s_j)$ is
path-connected if $s_l\to s_{l+1}$ for all $l<j$.

\emph{Winning histories.}
For a history $h=(s_1,\dots,s_j)\in\aS^{\leq k}$ with $j\geq 1$, say that $h$
is \emph{winning} if the controller has a strategy that always permits at
least one action and guarantees $\sem{\varphi}_t(w)=\top$ for all $t\geq k+1$
along every consistent infinite extension $w=s_1\cdots s_j s_{j+1}\cdots$;
write $W^j$ for the winning histories of length $j$. By
Equation~\eqref{eq:translation-invariance}, the obligations that remain along
a play of $\gG_\varphi$ after time $t$ depend only on the last $\min\{t,k\}$
states of the play: the windows still to be evaluated consist of these states
and of future states only. Hence the controller can still enforce $\varphi$
from a given play iff the history formed by its last $\min\{t,k\}$ states is
winning. For $h$ of length $j=k$, define
\begin{align*}
\xi^*(h)\coloneqq
\{a\in\aA \mid \forall s'\in\dT(s_k,a)\colon \sem{\varphi}_{k+1}(h,s')=\top
\,\wedge\,(s_2,\dots,s_k,s')\in W^k\},
\end{align*}
and for $h$ of length $j<k$, define
$\xi^*(h)\coloneqq\{a\in\aA \mid \forall s'\in\dT(s_j,a)\colon (h,s')\in W^{j+1}\}$.

\emph{Step 0: local characterisation.}
For every $h\in\aS^{\leq k}$, we claim that $h$ is winning iff
$\xi^*(h)\neq\emptyset$, and that $\xi^*(h)$ consists precisely of the actions
permitted at $h$ by some witnessing strategy. The arguments are those of
Lemma~\ref{lem:local-winning}: for the ``only if'' directions, fix a witness
$\xi$ and $a\in\xi(h)$; for $j<k$, the restriction of $\xi$ to extensions of
$(h,s')$ witnesses $(h,s')\in W^{j+1}$ for every successor $s'$, and for
$j=k$, the first obligation $\sem{\varphi}_{k+1}(h,s')=\top$ holds along every
consistent extension, while the strategy shifted by one position (deleting the
oldest state, exactly as in Lemma~\ref{lem:local-winning}) witnesses
$(s_2,\dots,s_k,s')\in W^k$. For the ``if'' directions, glue the successor
witnesses behind the action $a$, again as in Lemma~\ref{lem:local-winning}.

\emph{Step 1: invariant families.}
Call a family $X=(X_1,\dots,X_k)$ with $X_j\subseteq\aS^j$ \emph{invariant} if
every $h\in X_j$ admits an action $a$ such that all $a$-successor histories of
$h$ lie in $X$, where the $a$-successor histories of $h=(s_1,\dots,s_j)$ are
$\{(s_1,\dots,s_j,s') \mid s'\in\dT(s_j,a)\}$ if $j<k$ and
$\{(s_2,\dots,s_k,s') \mid s'\in\dT(s_k,a)\}$ if $j=k$, and where additionally
$\sem{\varphi}_{k+1}(s_1,\dots,s_k,s')=\top$ is required for all
$s'\in\dT(s_k,a)$ if $j=k$. By Step~0, the family $(W^1,\dots,W^k)$ is
invariant. Conversely, every invariant family consists of winning histories:
from $h\in X_j$, the strategy that, at every history whose last
$\min\{t,k\}$ states form an element of $X$, permits exactly the actions whose
$a$-successor histories remain in $X$ (and permits all actions elsewhere)
always permits an action by invariance, and along every consistent extension
all windows evaluated from time $k+1$ onwards satisfy $\varphi$ by the
$j=k$ clause and Equation~\eqref{eq:translation-invariance}.

\emph{Step 2: the algorithm computes the largest invariant family on
path-connected histories.}
After initialisation, $\xi^{\mathrm{alg}}(h)=\aA$ for $|h|<k$ and
$\xi^{\mathrm{alg}}(h)=\{a \mid \forall s'\in\dT(s_k,a)\colon
\sem{\varphi}_{k+1}(h,s')=\top\}$ for $|h|=k$. By inspection of
\textsc{PropagateConflicts}, an action $a$ is removed from a history $h$
exactly when some $a$-successor history of $h$ is popped from $\aU$, with one
exception: a pop of $(s_1,\dots,s_k)$ removes $a$ from a sliding predecessor
$(s',s_1,\dots,s_{k-1})$ only if $s'\to s_1$. This exception only concerns
histories that are not path-connected; moreover, a pop of a history that is
not path-connected never removes actions from a path-connected history, since
every removal branch requires the transition $s_{j-1}\xrightarrow{a}s_j$ for
the last link, and every predecessor it touches contains all internal links of
the popped history except possibly the last. Hence, restricted to
path-connected histories, removals happen exactly at pops of $a$-successor
histories, and histories are popped exactly when their action sets become
empty (Lines~\ref{alg:empty-check} and~\ref{alg:pred-empty-check}).
The induction of Step~3 in the proof of
Theorem~\ref{thm:hierarchical-iterative} applies verbatim: for every invariant
family $X$ of path-connected histories, no element of $X$ is ever popped and
no witnessing action is ever removed; conversely, the surviving family
$(\{h \mid |h|=j,\ \xi^{\mathrm{alg}}(h)\neq\emptyset\})_{j\leq k}$, restricted
to path-connected histories, is invariant. Hence, on path-connected histories,
the surviving family equals $(W^1,\dots,W^k)$ and, by the same intersection
computation, $\xi^{\mathrm{alg}}(h)=\xi^*(h)$ for every path-connected winning
history $h$.

\emph{Step 3: conclusion.}
The returned strategy
$\sigma(w_{1:t})\coloneqq\xi^{\mathrm{alg}}(w_{\max\{1,t-k+1\}:t})$ reads only
the last $\min\{t,k\}$ states, so it is $k$-history dependent, and every
history it is applied to along a play is path-connected. The history
$(\sinit)$ is winning iff $\gG_\varphi$ admits a winning strategy. If so, then
$\sigma$ is winning: along every $\sigma$-consistent play, the current history
remains in the surviving family, so $\sigma$ permits an action, and all
evaluated windows satisfy $\varphi$ by the $j=k$ clause. Moreover, $\sigma$
subsumes every winning strategy $\zeta$: if $a\in\zeta(w_{1:t})$ for a
$\zeta$-consistent play prefix, then, by Step~0 applied to the restrictions
and shifts of $\zeta$, the action $a$ lies in $\xi^*$ of the current history,
which equals $\sigma(w_{1:t})$. Hence $\sigma$ is maximally permissive.

\emph{Complexity.}
The initialisation inspects every length-$k$ history, action, and successor in
time $\mathcal O(|\aS|^{k}\cdot|\aA|\cdot|\aS|)$. During propagation, every
pair of a history and an action is removed at most once and every history is
popped at most once; each pop inspects at most $|\aA|\cdot(|\aS|+1)$
predecessor pairs. The total is $\mathcal O(|\aS|^{k+1}\cdot|\aA|)$.
\end{proof}

\reachabilityPruning*

\begin{proof}
Path-connected histories are closed under taking $a$-successor histories:
appending a successor state adds the link $s_j\to s'$, and, for $j=k$,
deleting the oldest state only removes a link. As observed in Step~2 of the
proof of Theorem~\ref{thrm:algo1}, a pop of a history that is not
path-connected never removes an action from a path-connected history.
Consequently, the run of Algorithm~\ref{algo:1} restricted to path-connected
histories---skipping all other histories in the initialisation loops and in
the predecessor sets---performs exactly the subsequence of operations of the
unrestricted run that touch path-connected histories, and therefore computes
the same action sets on all of them.

For the complexity, the number of path-connected histories of length $j$ is at
most $|\aS|\cdot B^{j-1}$, since each state has at most $B$ successors. The
initialisation inspects every path-connected length-$k$ history, action, and
successor, in time $\mathcal O(|\aS|\cdot B^{k-1}\cdot|\aA|\cdot B)$. During
propagation, each pop inspects its short prefix and at most $B$ sliding
predecessors $(s',s_1,\dots,s_{k-1})$ with $s'\to s_1$, for each of the
$|\aA|$ actions, and every history is popped at most once. The total is
$\mathcal O(|\aS|\cdot B^{k}\cdot|\aA|)$.
\end{proof}

\hardness*

\begin{proof}
Fix $k\geq 2$ and $\Delta\tau=1$. We construct a game in which the environment
chooses a bit, the play traverses a corridor of $k-1$ cells, and the
controller must announce the bit at the end of the corridor.

\emph{Construction.}
Let $\aS\coloneqq\{0,1,\dots,k+2\}\subseteq\ZN$, where $0$ and $1$ are the
\emph{bit} values, $2,\dots,k$ are the \emph{corridor} cells, and
$k+1$ and $k+2$ are the \emph{output} cells announcing bit $0$ and bit $1$,
respectively. Let $\aA\coloneqq\{0,1\}$ and define, for all $a\in\aA$,
\begin{align*}
\dT(b,a)&\coloneqq\{2\} \text{ for } b\in\{0,1\}, &
\dT(v,a)&\coloneqq\{v+1\} \text{ for } 2\leq v\leq k-1,\\
\dT(k,a)&\coloneqq\{k+1+a\}, &
\dT(k+1,a)&\coloneqq\dT(k+2,a)\coloneqq\{0,1\},
\end{align*}
with $\sinit\coloneqq k+1$. Thus every play has the shape
\begin{align*}
k{+}1,\;b_1,\;2,\dots,k,\;o_1,\;b_2,\;2,\dots,k,\;o_2,\;\dots,
\end{align*}
where each bit
$b_m\in\{0,1\}$ is chosen by the environment and each output
$o_m\in\{k+1,k+2\}$ is chosen by the controller via the action at cell $k$.
Let the \emph{bad} windows be
\begin{align*}
B\coloneqq\{(b,2,3,\dots,k,x)\in\aS^{k+1} \mid b\in\{0,1\},\ x\neq k+1+b\},
\end{align*}
and let $\varphi\coloneqq\varphi_P$ for $P\coloneqq\aS^{k+1}\setminus B$ be
the order-$k$ differential safety property given by
Lemma~\ref{lem:windows}(1).

\emph{Only the intended windows are constrained.}
By the transition structure, a window whose oldest entry is a bit value
$b\in\{0,1\}$ is necessarily followed by the full corridor $2,\dots,k$ and
then by an output cell; every other window has an oldest entry outside
$\{0,1\}$ and hence lies in $P$. Therefore, along any play, $\varphi$ is
violated at time $t$ iff $s_{t-k}=b\in\{0,1\}$ and $s_t\neq k+1+b$, i.e., iff
some output does not announce the bit chosen $k$ steps earlier.

\emph{A $k$-history dependent winning strategy exists.}
Let $\xi(h)\coloneqq\{b\}$ whenever the last $k$ states of $h$ are
$(b,2,\dots,k)$ with $b\in\{0,1\}$, and $\xi(h)\coloneqq\aA$ otherwise. The
strategy always permits an action, and whenever the controller chooses the
output, the bit lies exactly $k$ steps in the past and is therefore the oldest
entry of the current length-$k$ history, so the chosen output is $k+1+b$.
Hence no bad window ever occurs, and $\xi$ is a $k$-history dependent winning
strategy.

\emph{No $(k-1)$-history dependent winning strategy exists.}
Suppose $\xi$ is a $(k-1)$-history dependent winning strategy. Both finite
plays $\pi_b\coloneqq k{+}1,\,b,\,2,\dots,k$ for $b\in\{0,1\}$ are consistent
with $\xi$: the transition out of $\sinit=k+1$ is resolved by the environment,
which may choose either bit, and the corridor is deterministic. Since $\xi$ is
winning, $\xi(\pi_0)\neq\emptyset$, and since $\pi_0$ and $\pi_1$ share the
last $k-1$ states $(2,\dots,k)$, we have $\xi(\pi_0)=\xi(\pi_1)\eqqcolon P_0$.
Pick any $a\in P_0$ and set $b\coloneqq 1-a$. The extension
$\pi_b\cdot(k+1+a)$ is consistent with $\xi$, and its window at time $k+2$ is
$(b,2,\dots,k,\,k+1+a)\in B$ since $k+1+a\neq k+1+b$. As $\xi$ is winning, it
permits an action at every consistent finite play, so this violating play
extends to an infinite consistent path, contradicting that $\xi$ is winning.
\end{proof}

\subsection{Proofs of Section~\ref{sec:hierarchical-safety}}

\paragraph{Conventions.}
We use the conventions
$\aW_0\coloneqq\{()\}$ and $\xi_0(())\coloneqq\aA$ from
Algorithm~\ref{algo:hierarchical-synthesis}; with these conventions, the claims
of Lemma~\ref{lem:hierarchical-suffix} hold trivially for $i=1$ as well.

We first record a local characterisation of winning windows.

\begin{lemma}[Local characterisation]
\label{lem:local-winning}
For every level $i$ and every $\bar s=(s_1,\dots,s_i)\in\aS^i$,
\begin{align*}
    \bar s\in\aW_i
    \iff
    \sem{\varphi_i}_{i}(\bar s)=\top
    \;\text{ and }\;
    \xi_i(\bar s)\neq\emptyset .
\end{align*}
\end{lemma}

\begin{proof}
($\Rightarrow$) Let $\xi$ witness $\bar s\in\aW_i$. Since $\xi$ always permits
an action and every state--action pair has at least one successor, a consistent
infinite extension $w$ of $\bar s$ exists, and $\sem{\varphi_i}_{i}(w)=\top$
gives $\sem{\varphi_i}_{i}(\bar s)=\top$. Pick any $a\in\xi(\bar s)$ and
$s'\in\dT(s_i,a)$. Define a strategy $\xi'$ on extensions of
$(s_2,\dots,s_i,s')$ by
$\xi'((s_2,\dots,s_i,s')\,u)\coloneqq\xi((s_1,\dots,s_i,s')\,u)$ for all
$u\in\aS^*$, and $\xi'(v)\coloneqq\aA$ on all other histories. Then
$w'=s_2\cdots s_i s' s_{i+2}\cdots$ is consistent with $\xi'$ if and only if
$w=s_1 w'$ is consistent with $\xi$, and by
Equation~\eqref{eq:translation-invariance} the level-$i$ window of $w'$ at time
$t'\geq i$ equals the level-$i$ window of $w$ at time $t'+1\geq i+1$, which
satisfies $\varphi_i$. Hence $(s_2,\dots,s_i,s')\in\aW_i$ for every successor
$s'$, i.e., $a\in\xi_i(\bar s)\neq\emptyset$.

($\Leftarrow$) Let $a\in\xi_i(\bar s)$, and for every $s'\in\dT(s_i,a)$ fix a
strategy $\xi_{s'}$ witnessing $(s_2,\dots,s_i,s')\in\aW_i$. Define
$\xi(\bar s)\coloneqq\{a\}$ and
$\xi((s_1,\dots,s_i,s')\,u)\coloneqq\xi_{s'}((s_2,\dots,s_i,s')\,u)$ for all
$s'\in\dT(s_i,a)$ and $u\in\aS^*$, and $\xi(v)\coloneqq\aA$ elsewhere. Then
$\xi$ always permits an action. Along every consistent extension $w$ of
$\bar s$ we have $\sem{\varphi_i}_{i}(w)=\sem{\varphi_i}_{i}(\bar s)=\top$ by
assumption, and for $t\geq i+1$ the level-$i$ window of $w$ at $t$ equals the
level-$i$ window at $t-1\geq i$ of the path $s_2\cdots s_i s' s_{i+2}\cdots$,
which is consistent with $\xi_{s'}$ and hence satisfies $\varphi_i$. Therefore
$\bar s\in\aW_i$.
\end{proof}

\hierarchicalSuffix*

\begin{proof}
Let $(s_1,\dots,s_i)\in\aW_i$, witnessed by a strategy $\xi$. Define a strategy
$\xi'$ on extensions of $(s_2,\dots,s_i)$ by
$\xi'((s_2,\dots,s_i)\,u)\coloneqq\xi((s_1,\dots,s_i)\,u)$ for all $u\in\aS^*$,
and $\xi'(v)\coloneqq\aA$ on all other histories. Then $\xi'$ always permits an
action, and $w'=s_2\cdots s_i s_{i+1}\cdots$ is a consistent extension of
$(s_2,\dots,s_i)$ under $\xi'$ if and only if $w=s_1 w'$ is a consistent
extension of $(s_1,\dots,s_i)$ under $\xi$.

Fix such a $w'$ and a time $t'\geq i-1$, and set $t\coloneqq t'+1\geq i$. Since
$\xi$ is a witness,
$\sem{\varphi_i}_t(w)=\sem{\varphi_i}_{i}(w_{t-i+1},\dots,w_t)=\top$ by
Equation~\eqref{eq:translation-invariance}. The hierarchy condition
\eqref{eq:hierarchy} then yields
$\sem{\varphi_{i-1}}_{i-1}(w_{t-i+2},\dots,w_t)=\top$. Since $w_\tau=w'_{\tau-1}$
for $\tau\geq 2$, we have
$(w_{t-i+2},\dots,w_t)=(w'_{t'-i+2},\dots,w'_{t'})$, which is precisely the
level-$(i-1)$ window of $w'$ at time $t'$. Hence
$\sem{\varphi_{i-1}}_{t'}(w')=\top$ for all $t'\geq i-1$, so $\xi'$ witnesses
$(s_2,\dots,s_i)\in\aW_{i-1}$.

For the second claim, let $a\in\xi_i(s_1,\dots,s_i)$, i.e.,
$(s_2,\dots,s_i,s')\in\aW_i$ for every $s'\in\dT(s_i,a)$. By the first claim,
$(s_3,\dots,s_i,s')\in\aW_{i-1}$ for every such $s'$, and therefore
$a\in\xi_{i-1}(s_2,\dots,s_i)$ by definition.
\end{proof}

Next, we show that the transient pass of
Algorithm~\ref{algo:hierarchical-synthesis} is exact. Recall that
$\hat\xi_k(\bar s)=\xi_k(\bar s)$ for $\bar s\in\aW_k$ and
$\hat\xi_k(\bar s)=\emptyset$ otherwise.

\begin{lemma}[Transient correctness]
\label{lem:transient}
Let $\varphi_1,\dots,\varphi_k$ be a hierarchical differential safety
condition. For every $j\in\{1,\dots,k-1\}$ and every
$(s_1,\dots,s_j)\in\aS^j$: the controller has a strategy that always permits
at least one action and guarantees $\sem{\varphi_{\min\{t,k\}}}_t(w)=\top$ for
all $t\geq j$ along every consistent infinite extension
$w=s_1\cdots s_j s_{j+1}\cdots$ if and only if
$\hat\xi_j(s_1,\dots,s_j)\neq\emptyset$. Moreover,
$\hat\xi_j(s_1,\dots,s_j)$ consists precisely of the actions permitted at
$(s_1,\dots,s_j)$ by some such strategy.
\end{lemma}

\begin{proof}
We say that a history $h$ of length $j\leq k$ is \emph{transient-winning} if
the controller can guarantee, in the above sense,
$\sem{\varphi_{\min\{t,k\}}}_t(w)=\top$ for all $t\geq j$ along every
consistent infinite extension of $h$. For $j=k$ the obligations are
$\sem{\varphi_k}_t(w)=\top$ for all $t\geq k$, so, by definition of winning
windows, a history of length $k$ is transient-winning iff it lies in $\aW_k$,
i.e., iff $\hat\xi_k(h)\neq\emptyset$; and the two directions of the proof of
Lemma~\ref{lem:local-winning} show that the permitted actions of
witnessing strategies at $h\in\aW_k$ are exactly $\xi_k(h)=\hat\xi_k(h)$.

We proceed by downward induction on $j=k-1,\dots,1$. Fix
$h=(s_1,\dots,s_j)$. The obligations split into the obligation at time $j$,
namely $\sem{\varphi_j}_j(h)=\top$, which is a property of $h$ itself, and the
obligations at times $t\geq j+1$, which are exactly the obligations of the
extension viewed from the successor history $(h,s')$ of length $j+1$.

($\Leftarrow$) Let $a\in\hat\xi_j(h)$. Then $\sem{\varphi_j}_j(h)=\top$, and
for every $s'\in\dT(s_j,a)$ we have $\hat\xi_{j+1}(h,s')\neq\emptyset$, so by
the induction hypothesis (or the base case, if $j+1=k$) there is a strategy
$\xi_{s'}$ witnessing that $(h,s')$ is transient-winning. Define
$\xi(h)\coloneqq\{a\}$, $\xi((h,s')\,u)\coloneqq\xi_{s'}((h,s')\,u)$ for all
$s'\in\dT(s_j,a)$ and $u\in\aS^*$, and $\xi(v)\coloneqq\aA$ elsewhere. Then
$\xi$ always permits an action, and along every consistent extension the
obligation at time $j$ holds by assumption and the obligations at times
$t\geq j+1$ hold because the extension is consistent with some $\xi_{s'}$ from
$(h,s')$ onwards. Hence $h$ is transient-winning, with $a$ permitted.

($\Rightarrow$) Let $\xi$ witness that $h$ is transient-winning, and let
$a\in\xi(h)$. Since $\xi$ always permits an action and every state--action
pair has a successor, a consistent infinite extension of $h$ exists, so
$\sem{\varphi_j}_j(h)=\top$. For every $s'\in\dT(s_j,a)$, the restriction of
$\xi$ to extensions of $(h,s')$ (completed by $\aA$ elsewhere) witnesses that
$(h,s')$ is transient-winning, since the consistent extensions of $(h,s')$
under the restriction are exactly the consistent extensions of $h$ under $\xi$
that start with $s'$, and their obligations at times $t\geq j+1$ are among the
obligations guaranteed by $\xi$. By the induction hypothesis,
$\hat\xi_{j+1}(h,s')\neq\emptyset$ for every $s'$, hence $a\in\hat\xi_j(h)$ by
the definition on Line~\ref{alg2:transient-end}. This also shows that every
action permitted by some witnessing strategy lies in $\hat\xi_j(h)$, which
together with ($\Leftarrow$) establishes the ``moreover'' part.
\end{proof}

\hierarchicalIterative*

\begin{proof}
Throughout the proof, we write $\xi_i^{\mathrm{alg}}$ for the action sets
maintained by the algorithm and
$\aW_i^{\mathrm{alg}}\coloneqq\{\bar s\in\aC_i \mid
\xi_i^{\mathrm{alg}}(\bar s)\neq\emptyset\}$ for the set computed on
Line~\ref{alg2:winning}, while $\xi_i$ and $\aW_i$ denote the sets defined in
Section~\ref{sec:hierarchical-safety}.
We prove by induction on $i$ that, after iteration $i$,
$\aW_i^{\mathrm{alg}}=\aW_i$ and
$\xi_i^{\mathrm{alg}}(\bar s)=\xi_i(\bar s)$ for every $\bar s\in\aW_i$. The
base case $i=0$ holds by the conventions of the initialisation.

Fix $i\geq 1$ and assume the claim for level $i-1$, so that the sets used on
Lines~\ref{alg2:candidates} and~\ref{alg2:inherited} are exact. Call a set
$X\subseteq\aC_i$ \emph{invariant} if every $\bar s=(s_1,\dots,s_i)\in X$
admits an action $a\in\xi_{i-1}(s_2,\dots,s_i)$ with
$(s_2,\dots,s_i,s')\in X$ for every $s'\in\dT(s_i,a)$.

\emph{Step 1: $\aW_i$ is invariant.}
Let $\bar s\in\aW_i$. By Lemma~\ref{lem:local-winning},
$\sem{\varphi_i}_{i}(\bar s)=\top$ and there exists $a\in\xi_i(\bar s)$; by
Lemma~\ref{lem:hierarchical-suffix}, $(s_2,\dots,s_i)\in\aW_{i-1}$, so
$\bar s\in\aC_i$, and moreover $a\in\xi_{i-1}(s_2,\dots,s_i)$. By definition of
$\xi_i$, every successor window $(s_2,\dots,s_i,s')$ lies in $\aW_i$.

\emph{Step 2: every invariant set $X$ satisfies $X\subseteq\aW_i$.}
Given $\bar s_0\in X$, consider the strategy that, at every history whose
current length-$i$ window lies in $X$, permits exactly the actions
$a\in\xi_{i-1}$ of the suffix for which all successor windows remain in $X$;
invariance guarantees that at least one such action exists, and the strategy
permits all actions elsewhere. Along every consistent extension of $\bar s_0$,
a straightforward induction shows that the window at every time $t\geq i$ lies
in $X\subseteq\aC_i$ and hence satisfies $\varphi_i$ by
Equation~\eqref{eq:translation-invariance}. Thus $\bar s_0\in\aW_i$.

\emph{Step 3: the algorithm computes the largest invariant set.}
First, let $X$ be invariant. An induction on the sequence of pops and removals
shows that no window of $X$ is ever inserted into $\aU$ and that, for every
$\bar s\in X$, no action $a$ witnessing the invariance of $\bar s$ is ever
removed from $\xi_i^{\mathrm{alg}}(\bar s)$: such an $a$ passes the
initialisation test on Line~\ref{alg2:inherited} because all its successor
windows lie in $X\subseteq\aC_i$, and it is removed on
Line~\ref{alg2:remove} only if one of these successor windows is popped, which
never happens. In particular, every window of $X$ retains a nonempty action
set.

Conversely, consider the surviving windows $\aW_i^{\mathrm{alg}}$. A window
whose action set becomes empty is inserted
into $\aU$ (Line~\ref{alg2:empty} or the propagation loop) and eventually
popped, and its action set remains empty afterwards; hence surviving windows
are never popped. For $\bar s\in \aW_i^{\mathrm{alg}}$ and
$a\in\xi_i^{\mathrm{alg}}(\bar s)$,
the action $a$ passed initialisation, so $a\in\xi_{i-1}(s_2,\dots,s_i)$ and all
successor windows lie in $\aC_i$; and $a$ was never removed, so no successor
window was ever popped, i.e., all successor windows lie in
$\aW_i^{\mathrm{alg}}$. Hence $\aW_i^{\mathrm{alg}}$ is invariant.

Combining the three steps: $\aW_i$ is invariant (Step~1), so
$\aW_i\subseteq \aW_i^{\mathrm{alg}}$ (first part of Step~3); and
$\aW_i^{\mathrm{alg}}$ is invariant (second
part of Step~3), so $\aW_i^{\mathrm{alg}}\subseteq\aW_i$ (Step~2). Therefore
$\aW_i^{\mathrm{alg}}=\aW_i$. For the action sets, let $\bar s\in\aW_i$. The
argument above shows
\begin{align*}
    \xi_i^{\mathrm{alg}}(\bar s)
    &=
    \{a\in\xi_{i-1}(s_2,\dots,s_i)
    \mid
    \forall s'\in\dT(s_i,a)\colon (s_2,\dots,s_i,s')\in\aW_i^{\mathrm{alg}}\}
    \\
    &=
    \xi_{i-1}(s_2,\dots,s_i)\cap\xi_i(\bar s)
    =
    \xi_i(\bar s),
\end{align*}
where the last equality uses Lemma~\ref{lem:hierarchical-suffix}.

\emph{Maximal permissiveness.}
Restricted to $\aW_k$, the action sets $\xi_k$ form a winning strategy by
Step~2 applied to $X=\aW_k$. Moreover, if $\zeta$ is any winning strategy and
$a$ is permitted by $\zeta$ at a history with current window $\bar s\in\aW_k$,
then, as in the proof of Lemma~\ref{lem:local-winning}, the shifted strategy
witnesses that every successor window under $a$ is winning, so
$a\in\xi_k(\bar s)$. Hence, on every history whose current window lies in
$\aW_k$, $\xi_k$ subsumes every winning strategy and coincides with the
maximally permissive winning strategy of $\gG_{\varphi_k}$
obtained by direct synthesis.

\emph{Hierarchical winning condition.}
Consider the strategy $\sigma$ from the statement and suppose the turn-based
game with winning condition $\aF_{\varphi_1,\dots,\varphi_k}$ admits a winning
strategy. Restricting this strategy to extensions of its consistent length-$1$
play shows that $(\sinit)$ is transient-winning in the sense of
Lemma~\ref{lem:transient}, so $\hat\xi_1((\sinit))\neq\emptyset$. Along every
$\sigma$-consistent play $w$, a straightforward induction shows
$\hat\xi_t(w_{1:t})\neq\emptyset$ for all $t<k$ and
$w_{t-k+1:t}\in\aW_k$ for all $t\geq k$: each
$a\in\hat\xi_t(w_{1:t})$ forces all successor histories to retain a nonempty
set by Line~\ref{alg2:transient-end} (for $t<k-1$) or to lie in $\aW_k$ (for
$t=k-1$), and each $a\in\xi_k$ keeps the window in $\aW_k$. In particular,
$\sigma$ always permits an action, $\sem{\varphi_t}_t(w)=\top$ for all $t<k$
by the test on Line~\ref{alg2:transient-end}, and
$\sem{\varphi_k}_t(w)=\top$ for all $t\geq k$ by
Lemma~\ref{lem:local-winning} and
Equation~\eqref{eq:translation-invariance}; hence $\sigma$ is winning for
$\aF_{\varphi_1,\dots,\varphi_k}$. Conversely, let $\zeta$ be any winning
strategy for $\aF_{\varphi_1,\dots,\varphi_k}$ and let $a\in\zeta(w_{1:t})$
for a $\zeta$-consistent play prefix. For $t<k$, the restriction of $\zeta$ to
extensions of $w_{1:t}$ witnesses that $w_{1:t}$ is transient-winning with $a$
permitted, so $a\in\hat\xi_t(w_{1:t})=\sigma(w_{1:t})$ by the ``moreover''
part of Lemma~\ref{lem:transient}. For $t\geq k$, the remaining obligations
are $\sem{\varphi_k}_{t'}(w)=\top$ for $t'>t$ (together with the already
determined past), so, as in the maximal-permissiveness argument above,
$a\in\xi_k(w_{t-k+1:t})=\sigma(w_{1:t})$. Hence $\sigma$ subsumes $\zeta$ and
is the maximally permissive winning strategy for
$\aF_{\varphi_1,\dots,\varphi_k}$.

\emph{Complexity.}
At level $i$, the candidates are exactly the windows $(s_1,\bar w)$ with
$\bar w\in\aW_{i-1}$ and $s_1\in\aS$ that satisfy $\varphi_i$, so $\aC_i$ can
be enumerated in time $\mathcal O(|\aW_{i-1}|\cdot|\aS|)$ and
$|\aC_i|\leq|\aW_{i-1}|\cdot|\aS|$. The initialisation inspects every candidate,
inherited action, and successor, in time
$\mathcal O(|\aC_i|\cdot|\aA|\cdot|\aS|)$. During propagation, each pair of a
window and an action is removed at most once, each window is popped at most
once, and each pop inspects at most $|\aA|\cdot|\aS|$ predecessor pairs, giving
$\mathcal O(|\aC_i|\cdot|\aA|\cdot|\aS|)$ in total. Summing over the levels
yields $\mathcal O\bigl(\sum_{i=1}^{k}|\aC_i|\cdot|\aA|\cdot|\aS|\bigr)$, which
is bounded by $\mathcal O(|\aS|^{k+1}\cdot|\aA|)$ since $|\aC_i|\leq|\aS|^i$.
The transient pass inspects every history of length $j<k$, action, and
successor, in time $\mathcal O\bigl(\sum_{j<k}|\aS|^{j+1}\cdot|\aA|\bigr)$,
which is dominated by the bound above.
\end{proof}

\section{Further Experimental Evaluation}

\subsection{Experimental two-dimensional car model}
\label{app:2d-car-model}

The experimental evaluation uses a finite-state abstraction of a planar vehicle.
The model is intentionally simple: it keeps the transition relation small enough
for explicit-state shield synthesis, while preserving the two features that are
central for differential safety: the system has momentum-like history through
finite differences, and the controller must be robust against nondeterministic
successors.  It can be read as a grid abstraction of a car or mobile robot moving
through a rectangular workspace toward a wall, possibly with rectangular
obstacles in the interior.

\paragraph{States and actions.}
For parameters \(X,Y\in\mathbb{N}\), the physical state is the grid position
\[
  q=(x,y)\in \aS=\{0,\ldots,X\}\times\{0,\ldots,Y\}.
\]
The initial position used in the simulations is configurable; synthesis starts
from the all-zero history unless otherwise stated.  At every step the controller
selects a commanded displacement
\[
  u=(u_x,u_y)\in
  \aA=\{0,\ldots,V_x\}\times \{0,\ldots,V_y\}.
\]
Commands are nonnegative in both coordinates.  In particular, the vehicle is
not allowed to deliberately move backwards; the \(x\)-coordinate models progress
towards the wall, while the \(y\)-coordinate allows lateral motion in the grid.
The synthesis output is a permissive shield: for every winning history it stores
the subset of commands that are safe against all disturbances.

\paragraph{Nondeterministic transition relation.}
After the command \(u\) is chosen, the environment chooses an integer
disturbance
\[
  \delta=(\delta_x,\delta_y)\in D(u)
  =
  \{-\eta_x(u),\ldots,\eta_x(u)\}
  \times
  \{-\eta_y(u),\ldots,\eta_y(u)\},
\]
where
\[
  \eta_x(u)=\min\{\eta_{\max},u_x\},
  \qquad
  \eta_y(u)=\min\{\eta_{\max},u_y\}.
\]
The successor position is
\[
  q' = q + u + \delta .
\]
The dependence of \(D(u)\) on \(u\) ensures that a zero command has no negative
overshoot in that coordinate, and more generally that the realized displacement
in each coordinate remains nonnegative.  The nondeterminism in the game is
precisely this disturbance choice.  During synthesis, a command is retained only
if every \(q'\) induced by every \(\delta\in D(u)\) is safe and remains in the
winning region.  During simulation, one of the allowed commands is selected by a
concrete policy and the disturbance is sampled from the corresponding finite
disturbance set.

\paragraph{Position safety.}
The zeroth-order safety condition is a position predicate.  A position
\(q=(x,y)\) is position-safe if
\[
  0\le x\le X,\qquad 0\le y\le Y,\qquad x\le x_{\mathsf{wall}},
\]
and \(q\) is not contained in any obstacle rectangle.  If
\(\mathcal{O}\) is the set of configured rectangles
\[
  O=[x_{\min},x_{\max}]\times[y_{\min},y_{\max}]\cap\mathbb{Z}^2,
\]
then the full position-safe set is
\[
  S_0 =
  \{(x,y)\in \aS \mid x\le x_{\mathsf{wall}}\}
  \setminus \bigcup_{O\in\mathcal{O}} O .
\]
The wall is always active; obstacles are optional.  This is the direct collision
source of unsafety.

\paragraph{Higher-order safety.}
Let \(h=(q_t,q_{t-1},\ldots,q_{t-k})\) be the stored history, ordered from most
recent to oldest.  Since the sampling period is normalized to one, the discrete
order-\(d\) derivative at a prospective successor \(q_{t+1}\) is the backward
finite difference
\[
  \partial^d q_{t+1}
  =
  q_{t+1}
  +
  \sum_{i=1}^{d}(-1)^i \binom{d}{i} q_{t+1-i},
  \qquad d=1,\ldots,k .
\]
The experimental specifications use Euclidean balls as derivative-safe sets:
\[
  \|\partial^d q_{t+1}\|_2 \le r_d ,
  \qquad d=1,\ldots,k .
\]
The radii \(r_1,\ldots,r_k\) therefore bound speed, acceleration, jerk, and
higher finite differences.  A transition is locally safe only if the successor is
position-safe and all derivative bounds up to the current synthesis order hold.
Consequently, a history may be losing even when its current position is far from
a wall or obstacle: the vehicle may be moving too quickly, accelerating too
abruptly, or have too large a higher-order finite difference to guarantee future
collision avoidance under all disturbances.

\paragraph{Agent's policy for simulation.}
At each state, the synthesized shield defines a mask of admissible commands.
The agent's policy never selects commands outside this mask.
Among the admissible commands, it first restricts attention to those with
maximal $x$-velocity. It then samples a command $(v_x,v_y)$ according to
\[
  \Pr(v_x,v_y)
  \propto
  \exp\!\left(
    -\frac{\left|\sqrt{v_x^2+v_y^2}-v^\star\right|}{\temp}
  \right),
\]
where $v^\star$ is the target speed and $\temp>0$ is the temperature. Smaller
$\temp$ makes the policy concentrate on commands whose speed is closest to
$v^\star$, while larger $\temp$ makes the choice more diffuse among the
maximal-$v_x$ admissible commands.

To demonstrate a wide range of usage, we use different policies in the \texttt{Wall} and \texttt{Obstacle} instances. For the \texttt{Wall} instance, we use a relatively slow and conservative policy, with target velocity $v^\star = 5$ and temperature $\temp=0.1$. 
For the \texttt{Obstacle} instance, we use a faster and more stochastic policy, with a target velocity of $v^\star=8$ and temperature $\temp=1.15$.

\subsection{Synthesis cost}

In Tables~\ref{tab:synt-time-app} and \ref{tab:synth-mem-app}, we show further results comparing synthesis costs of the baseline, direct, and iterative algorithms.

\begin{table}[t]
\centering
\setlength{\tabcolsep}{1.2\tabcolsep}
\caption{Synthesis Cost: Time}
\label{tab:synt-time-app}
\begin{tabular}{lccc|rrr}
\toprule
\multirow{2}{*}{Benchmark}  & \multicolumn{3}{c}{Synth. Time (s)} & \multicolumn{3}{c}{Pruned States (\%)}\\
\cline{2-7}
  & Baseline & Direct & Iterative & $d = 3$ & $d = 4$ & $d = 5$\\
\midrule
$E^{(40, 45), (4, 4), 1}_{[4, 5, 8]}$ & 25.1 & \textbf{1.46} & 1.51 & 0.0\% & - & - \\
$E^{(45, 40), (4, 4), 1}_{[4, 5, 7]}$ & 23.4 & 1.6 & \textbf{1.54} & 0.0\% & - & - \\
$E^{(60, 65), (8, 4), 1}_{[4, 6, 7]}$ & TO & \textbf{4.15} & 4.2 & 0.0\% & - & - \\
$E^{(65, 60), (8, 4), 1}_{[4, 5, 7]}$ & TO & \textbf{3.95} & 4.06 & 0.0\% & - & - \\
$E^{(470, 110), (12, 6), 1}_{[5, 2, 8, 11]}$ & TO & 42 & \textbf{16.9} & 89.5\% & 89.6\% & - \\
$E^{(450, 120), (12, 6), 1}_{[5, 2, 8, 10]}$ & TO & 45.9 & \textbf{18.2} & 89.5\% & 89.7\% & - \\
$E^{(520, 120), (12, 6), 1}_{[5, 2, 9, 10]}$ & TO & 52.8 & \textbf{21} & 89.5\% & 89.7\% & - \\
$E^{(540, 120), (14, 6), 1}_{[6, 3, 9, 11]}$ & TO & TO & TO & 0.0\% & 0.0\% & - \\
$E^{(20, 22), (4, 2), 1}_{[3, 3, 6, 10, 18]}$ & 52.3 & \textbf{5.7} & 7.33 & 0.0\% & 0.0\% & 57.8\% \\
$E^{(22, 20), (4, 2), 1}_{[4, 4, 6, 7, 18]}$ & TO & 14.4 & \textbf{8.05} & 0.0\% & 57.8\% & 88.8\% \\
$E^{(22, 20), (4, 2), 1}_{[3, 3, 6, 10, 20]}$ & 54.4 & \textbf{5.97} & 7.31 & 0.0\% & 0.0\% & 57.8\% \\
$E^{(32, 20), (4, 2), 1}_{[3, 3, 4, 5, 20]}$ & 14.9 & 2.81 & \textbf{0.469} & 0.0\% & 73.3\% & 92.9\% \\
$E^{(70, 20), (6, 4), 1}_{[5, 2, 8, 10, 12]}$ & 9.2 & 2.23 & \textbf{0.871} & 88.5\% & 88.7\% & 88.2\% \\
$E^{(90, 20), (8, 2), 1}_{[5, 2, 8, 10, 12]}$ & 18.5 & 2.27 & \textbf{1.05} & 85.1\% & 85.3\% & 84.8\% \\
$E^{(110, 20), (6, 2), 1}_{[5, 2, 8, 10, 12]}$ & 23.2 & 2.8 & \textbf{1.21} & 85.2\% & 85.4\% & 84.8\% \\
$E^{(130, 20), (6, 2), 1}_{[5, 2, 8, 10, 12]}$ & 28.3 & 3.35 & \textbf{1.35} & 85.2\% & 85.4\% & 84.9\% \\
\bottomrule
\end{tabular}
\end{table}
\begin{table}[t]
\centering
\setlength{\tabcolsep}{1.2\tabcolsep}
\caption{Synthesis Cost: Memory}
\label{tab:synth-mem-app}
\begin{tabular}{lccc|rrr}
\toprule
\multirow{2}{*}{Benchmark}  & \multicolumn{3}{c}{Synth. Memory (MB)} & \multicolumn{3}{c}{Pruned States (\%)}\\
\cline{2-7}
  & Baseline & Direct & Iterative & $d = 3$ & $d = 4$ & $d = 5$\\
\midrule
$E^{(40, 45), (4, 4), 1}_{[4, 5, 8]}$ & 3175.2 & \textbf{269.3} & 270.7 & 0.0\% & - & - \\
$E^{(45, 40), (4, 4), 1}_{[4, 5, 7]}$ & 2639.1 & 258.3 & \textbf{257.2} & 0.0\% & - & - \\
$E^{(60, 65), (8, 4), 1}_{[4, 6, 7]}$ & TO & \textbf{599.9} & 613.3 & 0.0\% & - & - \\
$E^{(65, 60), (8, 4), 1}_{[4, 5, 7]}$ & TO & \textbf{606.7} & 610.3 & 0.0\% & - & - \\
$E^{(470, 110), (12, 6), 1}_{[5, 2, 8, 11]}$ & TO & 2284.8 & \textbf{891.2} & 89.5\% & 89.6\% & - \\
$E^{(450, 120), (12, 6), 1}_{[5, 2, 8, 10]}$ & TO & 2201.9 & \textbf{932.1} & 89.5\% & 89.7\% & - \\
$E^{(520, 120), (12, 6), 1}_{[5, 2, 9, 10]}$ & TO & 2676.9 & \textbf{1032.3} & 89.5\% & 89.7\% & - \\
$E^{(540, 120), (14, 6), 1}_{[6, 3, 9, 11]}$ & TO & TO & TO & 0.0\% & 0.0\% & - \\
$E^{(20, 22), (4, 2), 1}_{[3, 3, 6, 10, 18]}$ & 3730.6 & 696.1 & \textbf{679.9} & 0.0\% & 0.0\% & 57.8\% \\
$E^{(22, 20), (4, 2), 1}_{[4, 4, 6, 7, 18]}$ & TO & 1540.3 & \textbf{529.4} & 0.0\% & 57.8\% & 88.8\% \\
$E^{(22, 20), (4, 2), 1}_{[3, 3, 6, 10, 20]}$ & 4384.6 & 741.1 & \textbf{706.1} & 0.0\% & 0.0\% & 57.8\% \\
$E^{(32, 20), (4, 2), 1}_{[3, 3, 4, 5, 20]}$ & 1502.0 & 297.2 & \textbf{41.6} & 0.0\% & 73.3\% & 92.9\% \\
$E^{(70, 20), (6, 4), 1}_{[5, 2, 8, 10, 12]}$ & 645.5 & 176.4 & \textbf{59.2} & 88.5\% & 88.7\% & 88.2\% \\
$E^{(90, 20), (8, 2), 1}_{[5, 2, 8, 10, 12]}$ & 1273.6 & 203.8 & \textbf{78.2} & 85.1\% & 85.3\% & 84.8\% \\
$E^{(110, 20), (6, 2), 1}_{[5, 2, 8, 10, 12]}$ & 1642.2 & 259.3 & \textbf{97.5} & 85.2\% & 85.4\% & 84.8\% \\
$E^{(130, 20), (6, 2), 1}_{[5, 2, 8, 10, 12]}$ & 1872.2 & 296.9 & \textbf{115.9} & 85.2\% & 85.4\% & 84.9\% \\
\bottomrule
\end{tabular}
\end{table}

\subsection{Trace simulation results}
We show further results on the simulation of the \texttt{Wall} and \texttt{Obstacle} instances.

In Table~\ref{tab:wall-obstacle-order-sweep-violations}, we show the average number of safety violations for each shield and each order.
As expected, when the order of the shield is higher than or equal to the order of the safety property, there are no safety violations, and the number of violations for a given shielded agent increases with the number of orders considered in the hierarchical property, as the property becomes more restrictive.

\begin{table}
\centering
\setlength{\tabcolsep}{1.2\tabcolsep}
\caption{Average number of derivative-safety violations per trace.}
\label{tab:wall-obstacle-order-sweep-violations}
\begin{tabular}{lrrrr|rrrr}
\toprule
\multirow{2}{*}{Shield order} & \multicolumn{4}{c}{Wall} & \multicolumn{4}{c}{Obstacle} \\
\cline{2-5} \cline{6-9}
  & $d=1$ & $d=2$ & $d=3$ & $d=4$ & $d=1$ & $d=2$ & $d=3$ & $d=4$ \\
\midrule
$k=1$ & 0 & 2.50 & 3.55 & 5.45 & 0 & 5.15 & 9.40 & 12.30 \\
$k=2$ & 0 & 0 & 2.35 & 8.25 & 0 & 0 & 6.55 & 11.55 \\
$k=3$ & 0 & 0 & 0 & 0.75 & 0 & 0 & 0 & 4.80 \\
$k=4$ & 0 & 0 & 0 & 0 & 0 & 0 & 0 & 0 \\
\bottomrule
\end{tabular}
\end{table}

In Fig.~\ref{fig:pressure-app}, we show the same pressure graph as Fig.~\ref{fig:pressure-main}, this time for the \texttt{Obstacle} instance, which shows the same phenomenon: 

\begin{figure}[t]
    \centering
    \includegraphics[width=0.85\linewidth]{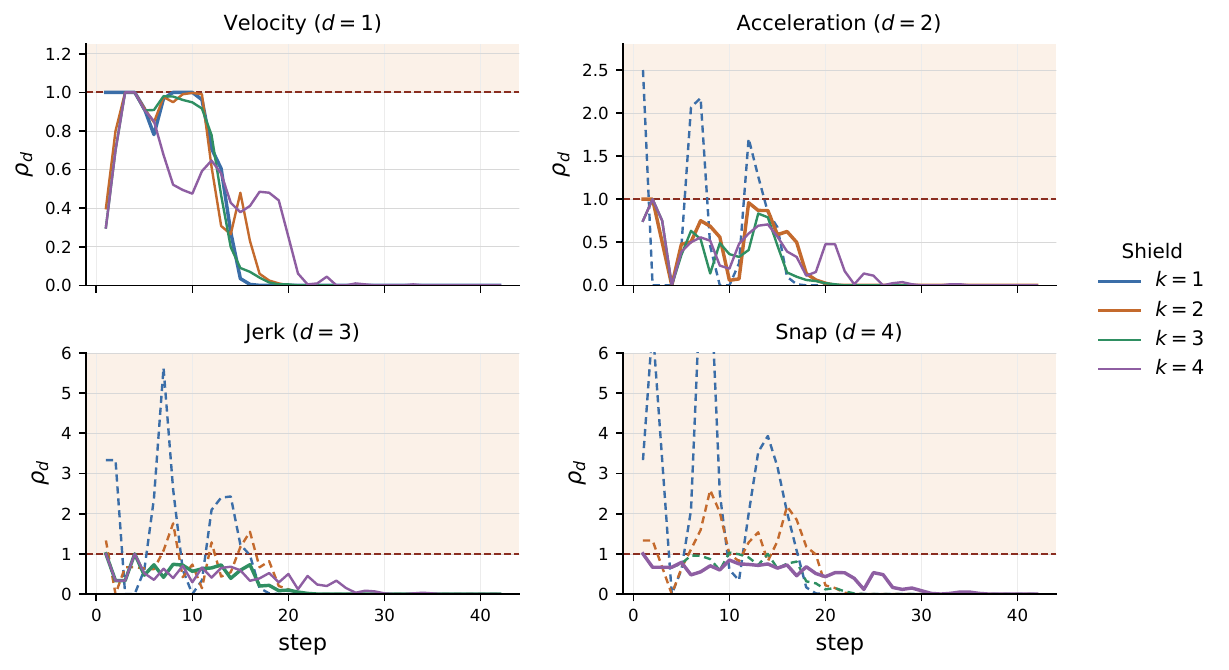}
    \caption{Derivative pressure by shield and constraint order for the \texttt{Obstacle} instance.}
    \label{fig:pressure-app}
\end{figure}

\paragraph{Derivative pressure heatmaps.}
In Figs.~\ref{fig:heatmap-app-obs} and \ref{fig:heatmap-app-wall}, we show heatmaps for the derivative pressure for different shield and constraint orders.
For each shield order $k$ and each constraint order $d$, the heatmaps summarize
the derivative pressure
$  \rho_d(t) = \frac{\|D^d x_t\|}{r_d}$.
The left heatmap reports the mean, over traces, of the maximum pressure reached
along the trace. Values below or equal to $1$ indicate that the corresponding
$d$-th order constraint remains satisfied throughout the trace, while values
above $1$ quantify the relative magnitude of the worst violation. The right
heatmap reports the average number of time steps per trace at which
$\rho_d(t)>1$. Thus entries on or below the diagonal are expected to be zero,
whereas entries above the diagonal measure how often lower-order shields fail to
enforce higher-order constraints.

\begin{figure}[t]
    \centering
    \includegraphics[width=0.85\linewidth]{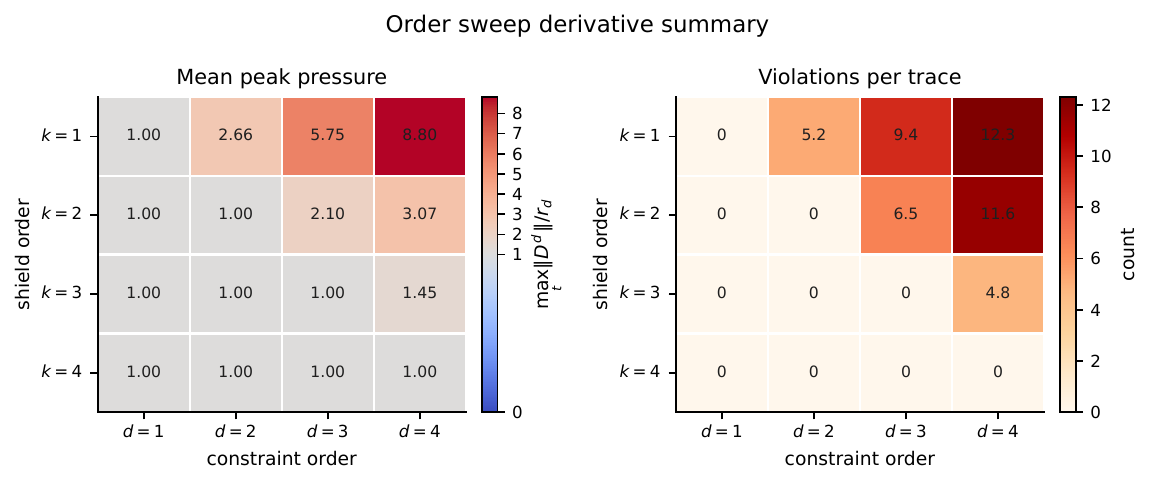}
    \caption{Heatmap of derivative pressure by shield constraint for the \texttt{Obstacle} instance.}
    \label{fig:heatmap-app-obs}
\end{figure}

\begin{figure}[t]
    \centering
    \includegraphics[width=0.85\linewidth]{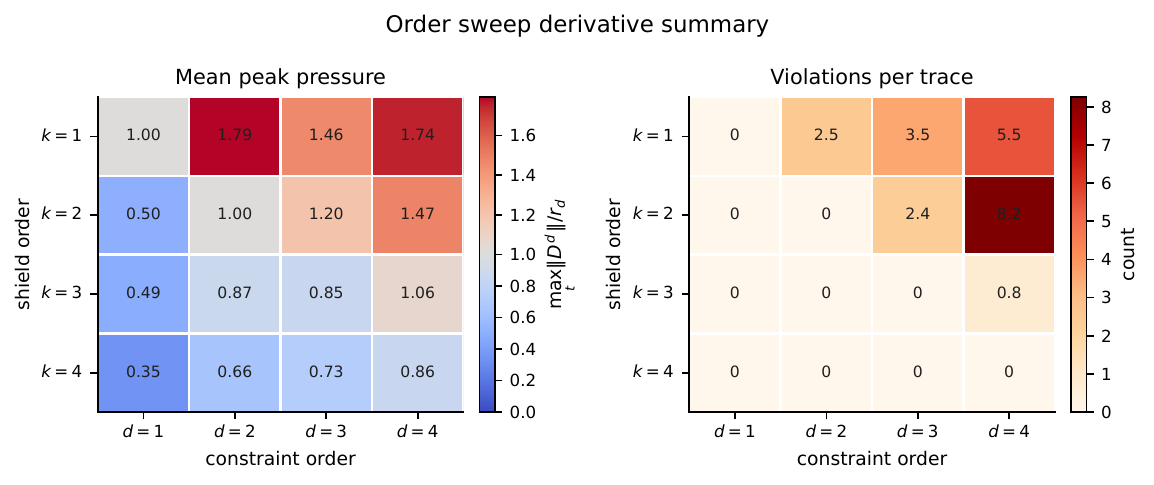}
    \caption{Heatmap of derivative pressure by shield constraint for the \texttt{Wall} instance.}
    \label{fig:heatmap-app-wall}
\end{figure}

\end{document}